\documentclass{article}

\usepackage[utf8]{inputenc}
\usepackage[T1]{fontenc}
\usepackage{textcomp}
\usepackage{bm}
\usepackage{dsfont}

\usepackage{amsmath}
\usepackage{amssymb}
\usepackage{amsthm}
\usepackage{amsfonts}       

\usepackage{mathtools}
\DeclarePairedDelimiter{\prn}{(}{)}
\DeclarePairedDelimiterX{\Prn}[2]{(}{)}{\,{#1} : {#2}\,}
\let\set\relax
\DeclarePairedDelimiter{\set}{\{}{\}}
\DeclarePairedDelimiterX{\Set}[2]{\{}{\}}{\,{#1} : {#2}\,}
\DeclarePairedDelimiter{\abs}{|}{|}
\DeclarePairedDelimiter{\norm}{\|}{\|}
\DeclarePairedDelimiter{\inpr}{\langle}{\rangle}

\DeclarePairedDelimiter{\brc}{[}{]}
\DeclarePairedDelimiterX{\Brc}[2]{[}{]}{\,{#1} : {#2}\,}

\DeclareFontFamily{U}{mathx}{}
\DeclareFontShape{U}{mathx}{m}{n}{<-> mathx10}{}
\DeclareSymbolFont{mathx}{U}{mathx}{m}{n}
\DeclareMathAccent{\widecheck}{0}{mathx}{"71}

\usepackage[svgnames]{xcolor}
\usepackage{graphicx}

\usepackage{booktabs}
\usepackage{multirow}

\usepackage{etoolbox}
\cslet{blx@noerroretextools}\empty%
\usepackage[date=year,eprint=false,doi=false,isbn=false,backend=biber,giveninits=true,maxcitenames=2,maxbibnames=99,natbib=true,url=false,sorting=ynt,style=apa,citestyle=authoryear-comp,backref=true,uniquelist=false]{biblatex}

\DeclareNameAlias{author}{last-first}
\DeclareSourcemap{
    \maps[datatype=bibtex]{
        \map{
            \step[fieldset=editor, null]
            \step[fieldset=series, null]
            \step[fieldset=language, null]
            \step[fieldset=address, null]
            \step[fieldset=location, null]
            \step[fieldset=month, null]
        }
    }
}

\DeclareFieldFormat[book]{apacase}{#1}

\DefineBibliographyStrings{english}{%
  backrefpage  = {Cited on page}, 
  backrefpages = {Cited on pages} 
}
\usepackage{algorithm}

\usepackage[colorlinks=true,citecolor=Navy,linkcolor=Maroon,urlcolor=Orchid,bookmarksnumbered,hypertexnames=false,pdfdisplaydoctitle,pdfusetitle,unicode]{hyperref}

\usepackage[noend]{algorithmic}

\usepackage{url}

\usepackage{upref}
\usepackage[capitalize,noabbrev]{cleveref}

\crefname{step}{Step}{Steps}
\Crefname{step}{Step}{Steps}

\usepackage{autonum}

\usepackage{thmtools}
\usepackage{thm-restate}
\usepackage{nicefrac}       
\usepackage{microtype}      
\usepackage{xspace}
\usepackage[color={red!100!green!33},colorinlistoftodos,prependcaption,textsize=small]{todonotes}
\usepackage{subcaption}
\usepackage{wrapfig}
\usepackage{siunitx}
\usepackage{xparse}
\allowdisplaybreaks

\usepackage{fullpage}
\usepackage{newtxtext,newtxmath}

\newtheorem{theorem}{Theorem}[section]
\newtheorem{lemma}[theorem]{Lemma}
\newtheorem{proposition}[theorem]{Proposition}

\newtheorem{assumption}[theorem]{Assumption}
\theoremstyle{definition}
\newtheorem{definition}[theorem]{Definition}

\newtheorem{remark}[theorem]{Remark}

\newcommand{\Z}{\mathbb{Z}}
\newcommand{\R}{\mathbb{R}}
\newcommand{\Rp}{\R_{\ge0}}

\newcommand{\zeros}{\mathbf{0}}
\newcommand{\ones}{\mathbf{1}}
\newcommand{\e}{\mathbf{e}}

\DeclareMathOperator{\argmin}{arg\,min}
\DeclareMathOperator{\argmax}{arg\,max}
\DeclareMathOperator{\conv}{conv}
\DeclareMathOperator{\dom}{dom}
\DeclareMathOperator{\intr}{int}

\DeclareMathOperator{\tr}{tr}
\DeclareMathOperator{\E}{\mathbb{E}}

\newcommand{\Ccal}{\mathcal{C}}

\newcommand{\Wcal}{\mathcal{W}}
\newcommand{\Xcal}{\mathcal{X}}
\newcommand{\Ycal}{\mathcal{Y}}
\newcommand{\ind}{\mathds{1}}

\newcommand{\x}{\bm{x}}
\newcommand{\y}{\bm{y}}
\newcommand{\U}{\bm{U}}
\newcommand{\V}{\bm{V}}
\newcommand{\W}{\bm{W}}

\newcommand{\thb}{\bm{\theta}}

\newcommand{\yhat}{\widehat{\y}_\Omega}

\newcommand{\yprd}{\widehat{\y}}
\newcommand{\ynep}{\y^*}
\newcommand{\yrnd}{\widetilde{\y}}

\newcommand{\CC}{C^2}
\newcommand{\Hs}{\mathsf{H}^\mathrm{s}}

\renewcommand{\S}{S}
\newcommand{\LT}[2]{L(#1;#2)}
\newcommand{\LS}[2]{\S(#1;#2)}
\newcommand{\LSOmega}[2]{\S_\Omega(#1;#2)}
\newcommand{\LSlog}[2]{\S_{\mathrm{logistic}}(#1;#2)}

\title{Online Structured Prediction with Fenchel--Young Losses and Improved Surrogate Regret for Online Multiclass Classification with Logistic Loss}

\author{%
Shinsaku Sakaue\\
The University of Tokyo\\
Tokyo, Japan\\
\href{mailto:sakaue@mist.i.u-tokyo.ac.jp}{sakaue@mist.i.u-tokyo.ac.jp}
\and
Han Bao\\
Kyoto University\\
Kyoto, Japan\\
\href{mailto:bao@i.kyoto-u.ac.jp}{bao@i.kyoto-u.ac.jp}
\and
Taira Tsuchiya\\
The University of Tokyo\\
Tokyo, Japan\\
\href{mailto:tsuchiya@mist.i.u-tokyo.ac.jp}{tsuchiya@mist.i.u-tokyo.ac.jp}
\and
Taihei Oki\\
The University of Tokyo\\
Tokyo, Japan\\
\href{mailto:oki@mist.i.u-tokyo.ac.jp}{oki@mist.i.u-tokyo.ac.jp}
}
\date{}

\begin{document}

\maketitle

\begin{abstract}%
  This paper studies online structured prediction with full-information feedback.
  For online multiclass classification, \Citet{Van_der_Hoeven2020-ug} established \emph{finite} surrogate regret bounds, which are independent of the time horizon, by introducing an elegant \emph{exploit-the-surrogate-gap} framework. 
  However, this framework has been limited to multiclass classification primarily because it relies on a classification-specific procedure for converting estimated scores to outputs. 
  We extend the exploit-the-surrogate-gap framework to online structured prediction with \emph{Fenchel--Young losses}, a large family of surrogate losses that includes the logistic loss for multiclass classification as a special case, obtaining finite surrogate regret bounds in various structured prediction problems. 
  To this end, we propose and analyze \emph{randomized decoding}, which converts estimated scores to general structured outputs. 
  Moreover, by applying our decoding to online multiclass classification with the logistic loss, we obtain a surrogate regret bound of $O(\norm{\U}_\mathrm{F}^2)$, where $\U$ is the best offline linear estimator and~$\norm{\cdot}_\mathrm{F}$ denotes the Frobenius norm. 
  This bound is tight up to logarithmic factors and improves the previous bound of $O(d\norm{\U}_\mathrm{F}^2)$ due to \Citet{Van_der_Hoeven2020-ug} by a factor of $d$, the number of classes.\looseness=-1 
\end{abstract}

\section{Introduction}\label{sec:intro}
Many machine learning problems involve predicting outputs in a finite set $\Ycal$ from input vectors in a vector space $\Xcal$.
A typical example is multiclass classification, and other tasks require predicting more complex structured objects, e.g., matchings and trees. 
Such problems, known as \emph{structured prediction}, are ubiquitous in many applications, including natural language processing and bioinformatics \citep{BakIr2007-kd}.
Since working directly on discrete output spaces is often intractable, it is usual to adopt the surrogate loss framework (e.g., \citet{Bartlett2006-gd}). 
Common examples are the logistic and hinge losses for classification. 
\Citet{Blondel2020-tu} have studied a family of \emph{Fenchel--Young losses}, which subsumes many practical surrogate losses for structured prediction; see \cref{subsec:fyloss} for details.

Structured prediction can be naturally extended to the online learning setting: for $t = 1,\dots,T$, an adversary picks $(\x_t, \y_t) \in \Xcal\times\Ycal$ and a learner plays $\yprd_t \in \Ycal$ given an input $\x_t \in \Xcal$.
The learner aims to minimize the cumulative target loss $\sum_{t=1}^T \LT{\yprd_t}{\y_t}$, where $L:\Ycal\times\Ycal\to\Rp$ is a target loss function, such as the 0-1 and Hamming losses. 
This paper focuses on the full-information setting, where the true output $\y_t \in \Ycal$ is available as feedback, while another common setting is the bandit setting, where only $\LT{\yprd_t}{\y_t}$ is given. 
A well-studied special case is online multiclass classification \citep{Fink2006-sx,Kakade2008-ud}. 
Let $\Ycal$ be a set representing $d$ classes, $\LT{\yprd_t}{\y_t}$ the standard 0-1 loss, $\U:\Xcal\to\R^d$ the best offline linear estimator, and $\S:\R^d\times \Ycal\to\Rp$ a surrogate loss (e.g., logistic or hinge) that measures the discrepancy between an estimated score vector in $\R^d$ and $\y_t \in \Ycal$. 
In this setting, a reasonable performance metric is the \emph{surrogate regret} $\mathcal{R}_T$ given by 
\begin{equation}
  \sum_{t=1}^T \LT{\yprd_t}{\y_t} = \sum_{t=1}^T \LS{\U\x_t}{\y_t} + \mathcal{R}_T.\footnote{In statistical learning, the term ``surrogate regret'' sometimes refers to the excess risk of surrogate losses, but we here use the term in the above sense following \Citet{Van_der_Hoeven2021-wi}.}
\end{equation}
The tenet behind this metric is that while we want to minimize the cumulative target loss, the best we can do in hindsight with the surrogate loss framework is to minimize the cumulative surrogate loss, and the surrogate regret accounts for the extra target loss incurred by actual plays, $\yprd_1,\dots,\yprd_T$.

For online multiclass classification with the logistic, hinge, and smooth hinge surrogate losses, \Citet{Van_der_Hoeven2020-ug} obtained an $O(d\norm{\U}_\mathrm{F}^2)$ surrogate regret bound, where $\norm{\cdot}_\mathrm{F}$ denotes the Frobenius norm.
Notably, the bound is independent of $T$, or \emph{finite}.
The core idea is to exploit the gap between the 0-1 and surrogate losses, which draws inspiration from online classification with abstention \citep{Neu2020-uw}. 
The author also gave an $O(dB\sqrt{T})$ surrogate regret bound for the bandit setting, where $B$ is the $\ell_2$-diameter of the domain containing $\U$. 
Later, \Citet{Van_der_Hoeven2021-wi} extended the idea to a more general feedback setting and also obtained lower bounds for the case of the smooth hinge surrogate loss. 
In the full-information setting, their lower bound is~$\Omega(dB^2)$, implying that the $O(d\norm{\U}_\mathrm{F}^2)$ bound for the smooth-hinge case is tight if $\norm{\U}_\mathrm{F} = \Theta(B)$.

However, the finite surrogate regret bound provided by \Citet{Van_der_Hoeven2020-ug} has been limited to online multiclass classification so far. 
Although the notion of surrogate regret naturally applies to more general structured prediction problems with the surrogate loss framework, the original exploit-the-surrogate-gap technique relies on a classification-specific decoding procedure for converting estimated scores in $\R^d$ to the outputs in $\set*{1,\dots,d}$, preventing the extension to structured prediction.
Since how to convert scores to structured outputs is non-trivial, it has been unclear when and how we can exploit the surrogate gap to obtain finite surrogate regret bounds in online structured prediction.\footnote{Applying \Citet{Van_der_Hoeven2020-ug} naively to $|\Ycal|$-class classification results in exponentially worse bounds in general.}

We extend the exploit-the-surrogate-gap framework to online structured prediction. 
Regarding surrogate losses, we consider a class of Fenchel--Young losses generated by \emph{Legendre-type} functions, due to its generality and useful properties (see \cref{subsec:fyloss}).
The main challenge lies in converting scores to the outputs in structured space $\Ycal$, for which we propose a \emph{randomized decoding} procedure (\cref{sec:randomizedq_decoding}), together with its efficient implementation based on a fast Frank--Wolfe-type algorithm (\cref{subsec:implementation}).
Our analysis of randomized decoding (\cref{lem:expected_target_bound}) reveals conditions of the structured output space, target loss, and surrogate loss under which we can obtain finite surrogate regret bounds by offsetting the regret in terms of surrogate losses with the surrogate gap.
Consequently, we establish finite surrogate regret bounds that hold in expectation and with high probability (\cref{thm:expected_regret_general,thm:high_probability_regret_general}). 
Additionally, \cref{thm:online-to-batch} shows that our randomized decoding enables \emph{online-to-batch conversion} of surrogate regret bounds to offline guarantees on the target risk. 

Although bounding the surrogate regret may seem to become easier by scaling up the surrogate loss relative to the target loss, our analysis of the surrogate regret is indeed sharp regardless of the scale of the surrogate loss.
To demonstrate the sharpness, \cref{section:multiclass} addresses online multiclass classification with the logistic loss, the same setting as that of \Citet{Van_der_Hoeven2020-ug}. 
We obtain an $O(\norm{\U}_\mathrm{F}^2)$ surrogate regret bound (\cref{thm:expected_regret_multi}), which improves the previous bound of $O(d\norm{\U}_\mathrm{F}^2)$ by a factor of $d$, the number of classes. 
We also provide an $\Omega({B^2}/{\ln^2 d})$ lower bound (\cref{thm:lower_bound}), implying that our bound is tight up to $\ln d$ factors under $\norm{\U}_\mathrm{F} = \Theta(B)$. 
These results shed light on an interesting $O(d)$ difference depending on surrogate losses: $O(d\norm{\U}_\mathrm{F}^2)$ is tight for the smooth hinge loss (\Citealp{Van_der_Hoeven2020-ug}; \Citealp{Van_der_Hoeven2021-wi}), while $O(\norm{\U}_\mathrm{F}^2)$ is almost tight for the logistic loss. 
Our work, grounded in sharp analysis, pushes the boundaries of the exploit-the-surrogate-gap framework and serves as a foundation for obtaining strong guarantees in online structured prediction.

\subsection{Additional Related Work}\label{subsec:related-work}
\paragraph{Structured prediction.}
We here present particularly relevant studies and defer a literature review to \cref{asec:related_work}. 
Prior to the development of the Fenchel--Young loss framework, \citet{Niculae2018-qg} studied \emph{SparseMAP} inference, which trades off the MAP and marginal inference so that the estimator captures the uncertainty, provides a unique solution, and is tractable.
SparseMAP regularizes the output by its squared $\ell_2$-norm and solves the resulting problem with optimization algorithms, such as Frank--Wolfe-type algorithms.
The fundamental idea of the Fenchel--Young losses is inherited from SparseMAP. 
To study the relationship between surrogate and target losses, or the \emph{Fisher consistency}, \citet{Ciliberto2016-nl,Ciliberto2020-zr} introduced target losses of the form $\LT{\y'}{\y} = \inpr{\y',\V\y}$ for some~$\V \in \R^{d \times d}$,\footnote{Although we focus on the Euclidean case, \citet{Ciliberto2016-nl,Ciliberto2020-zr} consider a more general form on Hilbert spaces.} termed as the \emph{Structure Encoding Loss Function (SELF)} by subsequent studies. 
They analyzed the regularized least-square decoder and obtained a \emph{comparison inequality}, a bound on the target excess risk in terms of the surrogate excess risk, for the squared loss. 
Since SELF encompasses many common target losses,\footnote{Indeed, any target loss on a finite set $\Ycal$ is written as a SELF with $\V \in \R^{|\Ycal|\times|\Ycal|}$ such that $V_{\y', \y} = \LT{\y'}{\y}$, though this representation ignores structural information of $\Ycal$ and typically causes inefficiency in learning.} the framework has been oftentimes leveraged by follow-up studies. 
For example, \citet{Blondel2019-gd} studied the Fisher consistency of Fenchel--Young losses based on projection for a generalized variant of SELF, which includes the 0-1, Hamming, and NDCG losses.

\paragraph{Online multiclass classification.}
For online binary classification on \emph{linearly separable} data, the classical Perceptron \citep{Rosenblatt1958-sh} achieves a finite surrogate regret bound. 
\Citet{Van_der_Hoeven2021-wi} extended this to multiclass classification, obtaining an $O(B^2)$ surrogate regret bound under the separability assumption, which matches the lower bound of \citet{Beygelzimer2019-si}.
By contrast, our $O(\norm{\U}_\mathrm{F}^2)$ surrogate regret bound applies to online multiclass classification with general non-separable data.
Another line of work has explored online logistic regression, where the performance is measured by the standard regret of the logistic loss.\footnote{A bound on the regret also upper bounds the surrogate regret in expectation since $1 - x \le -\log_2x$ for $x \in (0,1]$.} 
In this context, Online Newton Step \citep{Hazan2007-ta} is known as an $O(e^{B/2} \ln T)$-regret algorithm (omitting dimension factors), and obtaining an $O(\mathrm{poly}(B)\ln T)$ regret bound had been a major open problem \citep{McMahan2012-yb}. 
Despite a negative answer to the original question by \citet{Hazan2014-cc}, a seminal work by \citet{Foster2018-tm} achieved an $O(\ln(BT))$ regret bound (a doubly exponential improvement in $B$) via \emph{improper learning}, where a learner can use an estimator that is non-linear in $\x_t$.
While their original algorithm is inefficient, recent studies provide more efficient $O(B\ln T)$-regret improper algorithms \citep{Jezequel2021-hk,Agarwal2022-hk}. 
In contrast to this stream of research, we focus on obtaining finite surrogate regret bounds via proper learning.
For a more extensive literature review, we refer the reader to \Citet{Van_der_Hoeven2020-ug}.

\section{Preliminaries}
Let $\Rp$ be the set of non-negative reals. 
Let $[n] = \set{1,\dots,n}$ for any positive integer $n$.
Let $\ind_A$ be the 0-1 loss that takes one if $A$ is true and zero otherwise.
Let $\norm{\cdot}$ be any norm (typically, $\ell_1$ or $\ell_2$) that satisfies $\kappa\norm{\y} \ge \norm{\y}_2$ for some $\kappa > 0$ for any $\y \in \R^d$.
For a matrix $\W$, let $\norm{\W}_\mathrm{F} = \sqrt{\tr(\W^\top\W)}$ be the Frobenius norm.
Let $\ones$ be the all-ones vector and $\e_i$ the $i$th standard basis vector, i.e., all zeros except for the $i$th entry being one. 
For $\Ccal\subseteq \R^d$, $\conv(\Ccal)$ denotes its convex hull, $\intr(\Ccal)$ its interior, and $I_\Ccal:\R^d\to\set{0, +\infty}$ its indicator function, which takes zero if $\y \in \Ccal$ and $+\infty$ otherwise.
For $\Omega:\R^d\to\R\cup\set{+\infty}$, $\dom(\Omega) \coloneqq \Set*{\y \in \R^d}{\Omega(\y) < +\infty}$ denotes its effective domain and $\Omega^*(\thb) \coloneqq \sup\Set*{\inpr{\thb, \y} - \Omega(\y)}{\y \in \R^d}$ its convex conjugate.
Let $\triangle^d \coloneqq \Set{\y\in\Rp^d}{\norm{\y}_1=1}$ be the probability simplex and $\Hs(\y) \coloneqq -\sum_{i=1}^d y_i \ln y_i$ the Shannon entropy of $\y \in \triangle^d$.

Let $\Psi:\R^d\to\R\cup\set{+\infty}$ be a strictly convex function differentiable throughout $\intr(\dom\Psi) \neq \emptyset$. 
We say $\Psi$ is of \emph{Legendre-type} if $\lim_{i\to\infty}\norm{\nabla \Psi(\x_i)}_2 = +\infty$ whenever $\x_1,\x_2,\dots$ is a sequence in $\intr(\dom(\Psi))$ converging to a boundary point of $\intr(\dom(\Psi))$ (see, \citet[Section~26]{Rockafellar1970-sk}).\footnote{Strictly speaking, this property is \emph{essential smoothness}, which, combined with strict convexity, implies Legendre-type.} 
Also, given a convex set $\Ccal \subseteq \dom(\Psi)$, we say $\Psi$ is $\lambda$-strongly convex with respect to $\norm{\cdot}$ over $\Ccal$ if $\Psi(\y) \ge \Psi(\y') + \inpr{\nabla\Psi(\y'), \y - \y'} + \frac{\lambda}{2}\norm{\y - \y'}^2$ holds for any $\y \in \Ccal$ and $\y' \in \intr(\dom(\Psi)) \cap \Ccal$.

\subsection{Problem Setting}\label{subsec:problem_setting}
Let $\Xcal$ be an input vector space. 
For consistency with \citet{Blondel2020-tu}, we let $\Ycal$ be the set of outputs embedded into $\R^d$ in the standard manner. 
For example, we let $\Ycal = \set*{\e_1,\dots,\e_d}$ in multiclass classification with $d$ classes. 
We focus on the case where observable feedback comes from $\Ycal$.

As with online multiclass classification, we consider learning a linear estimator $\W$ that maps an input vector $\x \in \Xcal$ to a score vector $\W\x \in \R^d$. 
The learning proceeds for $t = 1,\dots,T$. 
In each $t$th round, an adversary picks an input vector $\x_t \in \Xcal$ and the true output $\y_t \in \Ycal$. 
The learner receives $\x_t$ and computes a score vector $\thb_t = \W_t\x_t$ with a current estimator $\W_t$.
The learner then chooses $\yprd_t \in \Ycal$ based on $\thb_t$, plays it, and incurs a target loss of $\LT{\yprd_t}{\y_t}$.
The learner receives $\y_t$ as feedback and updates $\W_t$ to $\W_{t+1}$. 
The goal of the learner is to minimize the cumulative target loss $\sum_{t=1}^T \LT{\yprd_t}{\y_t}$. 
We assume the following conditions on the output space and the target loss.
\begin{assumption}\label[assumption]{assump:nu}
  (I) There exists $\nu > 0$ such that $\norm{\y - \y'} \ge \nu$ holds for any $\y, \y' \in \Ycal$ with $\y \neq \y'$, 
  (II) for each $\y \in \Ycal$, the target loss $\LT{\cdot}{\y}$ is defined on $\conv(\Ycal)$, non-negative, and affine in the first argument,\footnote{
    Condition (II) is assumed for technical convenience, although target losses are inherently defined on $\Ycal$. 
    This specifically ensures $\E[\LT{\yrnd}{\y}] = \LT{\E[\yrnd]}{\y}$ for $\yrnd$ drawn randomly from $\Ycal$, which we will use in the proof of \cref{lem:expected_target_bound}.} and 
  (III) $\LT{\y'}{\y} \le \gamma\norm{\y' - \y}$ holds for some $\gamma > 0$, for any $\y' \in \conv(\Ycal)$ and $\y \in \Ycal$.
\end{assumption}  
These conditions are not restrictive; see \cref{subsec:example} for examples satisfying them. 
Regarding~(I), $\nu$ lower bounded in many cases. 
For instance, if $\norm{\cdot}$ is an $\ell_p$-norm and $\Ycal \subseteq \Z^d$, $\nu \ge 1$ holds. 
In addition, if $\y^\top\ones$ is constant for all $\y \in \Ycal$, distinct $\y, \y' \in \Ycal$ have at least two entries that differ by at least~$1$ in magnitude, hence $\nu \ge 2^{1/p}$.
As for (II), SELFs $\LT{\y'}{\y}$ $= \inpr{\y',\V\y}$ are defined on $\conv(\Ycal)$ and affine in $\y'$. 
Moreover, \citet[Appendix~A]{Blondel2019-gd} provides many target losses expressed as $\LT{\y'}{\y}$ $= \inpr{\y', \V\y + \bm{b}} + c(\y)$ for $\y, \y' \in \conv(\Ycal)$ with some $\V \in \R^{d\times d}$, $\bm{b}\in \R^d$ and $c(\y) \in \R$, which are also defined on $\conv(\Ycal)$ and affine in $\y'$. 
Condition (III) is typically satisfied by moderate~$\gamma$ values (see \cref{subsec:example}).
Note that the non-negativity and (III) imply $\LT{\y'}{\y} = 0$ if $\y' = \y$.

\subsection{Fenchel--Young Loss}\label{subsec:fyloss}
We adopt the surrogate loss framework considered in \citet{Blondel2020-tu}.
We define an intermediate \emph{score space} $\R^d$ between $\Xcal$ and $\Ycal$ and measure the discrepancy between a score vector $\thb \in \R^d$ and the ground truth $\y \in \Ycal$ with a surrogate loss $\S:\R^d\times\Ycal\to\Rp$; here, we suppose $\thb$ to be given by $\W_t\x_t$ as in \cref{subsec:problem_setting}.
\Citet{Blondel2020-tu} provides a general recipe for designing various surrogate losses, called \emph{Fenchel--Young losses}, for structured prediction from regularization functions. 
\begin{definition}
  Let $\Omega:\R^d\to\R\cup\set{+\infty}$ be a regularization function such that $\Ycal \subseteq \dom(\Omega)$. 
  The Fenchel--Young loss $\S_\Omega:\dom(\Omega^*)\times\dom(\Omega)\to\Rp$ generated by $\Omega$ is defined as 
  \[
    \LSOmega{\thb}{\y} \coloneqq \Omega^*(\thb) + \Omega(\y) - \inpr{\thb, \y}. 
  \]
\end{definition}
By definition, $\LSOmega{\thb}{\y}$ is convex in $\thb$ for any $\y \in \dom(\Omega)$.
Furthermore, $\LSOmega{\thb}{\y} \ge 0$ follows from the Fenchel--Young inequality, and $\LSOmega{\thb}{\y} = 0$ holds if and only if $\y \in \partial\Omega^*(\thb)$.

We focus on special Fenchel--Young losses studied in \citet[Section~3.2]{Blondel2020-tu}, which are generated by $\Omega$ of the form $\Psi + I_{\conv(\Ycal)}$ (i.e., $\Omega$ is the restriction of $\Psi$ to $\conv(\Ycal)$), where $\Psi$ is differentiable, of Legendre-type, and $\lambda$-strongly convex w.r.t.\ $\norm{\cdot}$, and satisfies $\conv(\Ycal) \subseteq \dom(\Psi)$ and $\dom(\Psi^*) = \R^d$. 
Such Fenchel--Young losses subsume various useful surrogate losses, including the logistic, CRF, and SparseMAP losses, and enjoy the following helpful properties.
See \citet[Propositions~2~and~3]{Blondel2020-tu} for more details, and also \cref{asec:crf} for a note on the CRF loss.
\begin{proposition}\label[proposition]{prop:fyloss_properties}
  Let $\S_\Omega$ be a Fenchel--Young loss generated by $\Omega = \Psi + I_{\conv(\Ycal)}$, where $\Psi:\R^d\to\R\cup\set{+\infty}$ satisfies the above properties. 
  For $\thb \in \R^d$, define the regularized prediction function as
  \[
    \yhat(\thb) \coloneqq \argmax\Set{\inpr{\thb, \y} - \Omega(\y)}{\y \in \R^d} = \argmax\Set{\inpr{\thb, \y} - \Psi(\y)}{\y \in \conv(\Ycal)}, 
  \]
  where the maximizer is unique. 
  Then, for any $\y \in \Ycal$, $\LSOmega{\thb}{\y}$ is differentiable in $\thb$ and the gradient is the residual, i.e., $\nabla \LSOmega{\thb}{\y} = \yhat(\thb) - \y$. 
  Furthermore, $\LSOmega{\thb}{\y} \ge \frac{\lambda}{2}\norm{\y - \yhat(\thb)}^2$ holds.\footnote{\Citet[Proposition~3]{Blondel2020-tu} shows $\LSOmega{\thb}{\y} \ge B_\Psi(\y \,\|\, \yhat(\thb))$, where $B_\Psi$ is the Bregman divergence induced by $\Psi$, and $B_\Psi(\y \,\|\, \yhat(\thb)) \ge \frac{\lambda}{2}\norm{\y - \yhat(\thb)}^2$ follows from the $\lambda$-strong convexity of $\Psi$ with respect to $\norm{\cdot}$. 
  The same inequality is also used in \citet[Lemma~3]{Blondel2019-gd}.
  This is the only part where we need $\Psi$ to be of Legendre-type.}
\end{proposition}
The last inequality will turn out useful in analyzing our randomized decoding (see \cref{lem:expected_target_bound}).
\Cref{prop:fyloss_properties} also implies $\norm{\nabla \LSOmega{\thb}{\y}}^2 \le \frac{2}{\lambda}\LSOmega{\thb}{\y}$, which we will use in the proof of \cref{thm:expected_regret_general}. 
This type of inequality plays a crucial role in exploiting the surrogate gap, as highlighted in \cref{prop:surrogate_gap}.

\subsection{Examples}\label{subsec:example}
Below are three typical structured prediction problems and Fenchel--Young losses satisfying the above conditions, and \Cref{asec:additional_applications} gives two more examples; all the five are considered in \citet[Section~4]{Blondel2019-gd}. 
More examples of structured outputs are provided in \citet[Section~7.3]{Blondel2020-tu}.

\paragraph{Multiclass classification.}
Let $\Ycal = \set*{\e_1,\dots,\e_d}$ and $\norm{\cdot}$ be the $\ell_1$-norm.  
Since $\norm{\e_i - \e_j}_1 \ge 2$ holds for any distinct $i, j \in [d]$, we have $\nu = 2$.
For any $\e_i \in \Ycal$, the 0-1 loss, $\LT{\y'}{\e_i} = \ind_{\y'\neq\e_i}$, can be extended on $\conv(\Ycal)$ as $\LT{\y'}{\e_i} = \inpr{\y', \ones - \e_i}$, which is affine in $\y'$ and equals $\sum_{j \neq i} y'_j = \frac{1}{2}\prn*{1 - y'_i + \sum_{j \neq i} y'_j} = \frac{1}{2}\norm{\e_i - \y'}_1$ due to $\sum_{i=1}^n y'_i = 1$, hence $\gamma = \frac12$.
As detailed in \cref{section:multiclass}, the logistic loss can be written as a Fenchel--Young loss generated by an entropic regularizer $\Omega$.\footnote{The multiclass hinge loss \citep{Crammer_undated-db} is also written as a \emph{cost-sensitive} Fenchel--Young loss \citep[Section~3.4]{Blondel2020-tu}. 
However, this requires the ground truth $\y$ when computing a counterpart of $\yhat(\thb)$, which we need before $\y_t$ is revealed (see \cref{alg:sgaptron}). 
Thus, it seems difficult to apply our approach to the (smooth) hinge~loss.\label{footnote:cost-sensitive}}

\paragraph{Multilabel classification.}
We consider multilabel classification with $\Ycal = \set*{0, 1}^d$. 
If $\norm{\cdot}$ is the $\ell_2$-norm, we have $\nu = 1$. 
A common target loss is the Hamming loss $\LT{\y'}{\y} = \frac{1}{d}\sum_{i=1}^d \ind_{y'_i \neq y_i}$, where the division by $d$ scales the loss to $[0,1]$. 
For any $\y \in \Ycal$, it is represented on $\conv(\Ycal)$ as $\LT{\y'}{\y} = \frac{1}{d}\prn*{\inpr{\y', \ones} + \inpr{\y, \ones} - 2\inpr{\y', \y}}$, which is affine in $\y'$ and satisfies $\LT{\y'}{\y} = \frac{1}{d}\norm{\y' - \y}_1 \le \frac{1}{\sqrt{d}}\norm{\y' - \y}_2$, hence $\gamma = \frac{1}{\sqrt{d}}$. 
If we let $\Omega = \frac{1}{2}\norm{\cdot}_2^2 + I_{\conv(\Ycal)}$, we have $\lambda = 1$ (i.e., $1$-strongly convex), and the resulting Fenchel--Young loss is the SparseMAP loss: $\LSOmega{\thb}{\y} = \frac12\norm{\y - \thb}_2^2 - \frac12\norm{\yhat(\thb) - \thb}_2^2$.

\paragraph{Ranking.}
  We consider predicting the ranking of $n$ items. 
  Let $d = n^2$ and $\Ycal \subseteq \set{0,1}^{d}$ be the set of all $n\times n$ permutation matrices, vectorized into $\set{0,1}^d$. 
  Then, $\conv(\Ycal)$ is the Birkhoff polytope. 
  For $\y \in \conv(\Ycal)$, $y_{ij}$ refers to the $(i, j)$ entry of the corresponding matrix.
  If $\norm{\cdot}$ is the $\ell_1$-norm, $\norm{\y - \y'}_1 \ge 4$ holds for distinct $\y, \y'\in \Ycal$, hence $\nu = 4$.
  We use a target loss that counts mismatches. 
  Specifically, let 
  $\LT{\y'}{\y} = \frac{1}{n}\sum_{i=1}^n \ind_{y'_{i, j_i}\neq y_{i, j_i}}$ for $\y, \y'\in \Ycal$, where $j_i \in [n]$ is a unique index such that $y_{ij_i} = 1$ for each $i\in [n]$ and the division by $n$ scales the loss to $[0,1]$. 
  For any $\y \in \Ycal$, this loss can be represented on $\conv(\Ycal)$ as $\LT{\y'}{\y} = \frac{1}{n} \inpr{\y', \ones - \y}$, which is affine in $\y'$. 
  Furthermore, it equals $\frac{1}{2n}\sum_{i=1}^n (1 - y'_{ij_i} + \sum_{j\neq j_i}y'_{ij}) = \frac{1}{2n}\norm{\y' - \y}_1$ since $\sum_{j}y'_{ij} = 1$ holds for each $i\in[n]$, hence $\gamma = \frac{1}{2n}$.
  Drawing inspiration from celebrated entropic optimal transport \citep{Cuturi2013-gy}, we consider a Fenchel--Young loss generated by $\Omega = -\frac{1}{\mu}\Hs + I_{\conv(\Ycal)}$, where  $\mu > 0$ controls the regularization strength.\footnote{With this choice of $\Omega$, we can efficiently compute the regularized prediction $\yhat(\thb)$ with the Sinkhorn algorithm.} 
  Since $-\frac{1}{\mu}\Hs$ is $\frac{1}{n\mu}$-strongly convex w.r.t. $\norm{\cdot}_1$ over $\conv(\Ycal)$ \citep[Proposition~2]{Blondel2019-gd}, we have $\lambda = \frac{1}{n\mu}$. 
  The resulting $\LSOmega{\thb}{\y}$ is written as $\inpr{\thb, \yhat(\thb) - \y} + \frac{1}{\mu} \Hs(\yhat(\thb))$, where the first term measures the affinity between $\thb$ and $\y \in \Ycal$, and the second term penalizes the uncertainty of $\yhat(\thb)$.\looseness=-1 

  \section{Randomized Decoding}\label{sec:randomizedq_decoding}

We present our key technical tool, randomized decoding, for converting a score vector $\thb \in \R^d$ to an output $\yprd \in \Ycal$. 
Our randomized decoding (\cref{alg:decoding}) returns either $\ynep \in \Ycal$ closest to $\yhat(\thb) \in \conv(\Ycal)$ or random $\yrnd \in \Ycal$ such that $\E[\yrnd \,|\, Z=1] = \yhat(\thb)$, where $\Omega$ is a regularization function generating the Fenchel--Young loss $\S_\Omega$ and $Z$ is the Bernoulli random variable with parameter $p$, as in \cref{step:bernoulli}. 
Intuitively, the closer the regularized prediction $\yhat(\thb)$ is to $\ynep$ (i.e., $\Delta^*$ is smaller), the more confident $\thb$ is about $\ynep$, and hence the decoding procedure returns $\ynep$ with a higher probability; 
otherwise, $\thb$ is not sufficiently confident about any $\y \in \Ycal$, and hence the decoding procedure more likely returns a random $\yrnd$.
The confidence is quantified by $2\Delta^*/\nu$ (smaller values indicate higher confidence), where $\nu$ is the minimum distance between distinct elements in $\Ycal$, as in \cref{assump:nu}.\footnote{This confidence measure is based on a rationale that $2\Delta^*/\nu < 1$ ensures that $\ynep$ is the closest point to $\yhat(\thb)$ among all $\y \in \Ycal$. Specifically, if $\Delta^* = \norm{\ynep - \yhat(\thb)} < \nu/2$ holds, for any $\y \in \Ycal\setminus\set{\ynep}$, the triangle inequality implies $\norm{\y - \yhat(\thb)} \ge \norm{\y - \ynep} - \norm{\ynep - \yhat(\thb)} > \nu - \nu /2 > \norm{\ynep - \yhat(\thb)}$. Note that the opposite is not always true.} 

\begin{algorithm}[tb]
  \caption{Randomized decoding $\psi_\Omega$}
  \label[algorithm]{alg:decoding}
  \begin{algorithmic}[1]
    \REQUIRE{$\thb \in \R^d$} 
    \STATE $\yhat(\thb) \gets \argmax\Set*{\inpr{\thb, \y} - \Psi(\y)}{\y\in\conv(\Ycal)}$ \label[step]{step:regularized_max}
    \STATE $\ynep \gets \argmin\Set*{\norm{\y - \yhat(\thb)}}{\y \in \Ycal}$ (breaking ties arbitrarily)\label[step]{step:nearest_extreme_point}
    \STATE $\Delta^* \gets \norm{\ynep - \yhat(\thb)}$ and $p \gets \min\set*{1, 2\Delta^*/\nu}$ 
    \STATE $Z\gets 0$ with probability $1 - p$; $Z \gets 1$ with probability $p$ \label[step]{step:bernoulli}
    \STATE $\yprd \gets \begin{cases} \ynep & \text{if $Z=0$} \\ \yrnd & \text{if $Z=1$, where $\yrnd$ is randomly drawn from $\Ycal$ so that $\E[\yrnd \,|\, Z=1] = \yhat(\thb)$} \end{cases}$   \label[step]{step:random_select}
    \STATE \textbf{return} $\psi_\Omega(\thb) = \yprd$
  \end{algorithmic}
\end{algorithm}

The following \cref{lem:expected_target_bound} is our main technical result regarding the randomized decoding.
Despite the simplicity of the proof, it plays a crucial role in the subsequent analysis of the surrogate regret.
\begin{lemma}\label[lemma]{lem:expected_target_bound}
  For any $(\thb, \y) \in \R^d\times\Ycal$, the randomized decoding $\psi_\Omega$ (\cref{alg:decoding}) satisfies
  \[
    \E[\LT{\psi_\Omega(\thb)}{\y}] \le \frac{4\gamma}{\lambda\nu} \LSOmega{\thb}{\y}. 
  \]
\end{lemma}
\begin{proof}
  Let $\Delta = \norm{\y - \yhat(\thb)}$. 
  Note that $\Delta \ge \Delta^* = \norm{\ynep - \yhat(\thb)}$ holds by definition of $\ynep$. 
  Since $\LT{\cdot}{\y}$ is affine as in \cref{assump:nu}, we have $\E[\LT{\yrnd}{\y} \,|\, Z=1] = \LT{\yhat(\thb)}{\y}$.
  Thus, it holds that
  \begin{align}
    \E[\LT{\psi_\Omega(\thb)}{\y}] 
    ={}& (1 - p) \LT{\ynep}{\y} + p \LT{\yhat(\thb)}{\y} \\ 
    ={}&
    \begin{cases}
      p \LT{\yhat(\thb)}{\y} & \text{if } \Delta^* \ge \nu/2 \text{ or } \ynep = \y, \\
      \prn*{1 - p} \LT{\ynep}{\y} + p \LT{\yhat(\thb)}{\y} & \text{if } \Delta^* < \nu/2 \text{ and } \ynep \neq \y.
    \end{cases}
  \end{align}
  Below, we will prove $\E[\LT{\psi_\Omega(\thb)}{\y}] \le 2\gamma\Delta^2/\nu$; 
  then, the desired bound follows from $\lambda\Delta^2/2 \le  \LSOmega{\thb}{\y}$ given in \cref{prop:fyloss_properties}.
  In the first case, from $p \le {2\Delta^*}/{\nu} \le {2\Delta}/{\nu}$ and $\LT{\yhat(\thb)}{\y} \le \gamma \Delta$ (due to \cref{assump:nu}), we obtain $p\LT{\yhat(\thb)}{\y} \le  {2\gamma\Delta^2}/{\nu}$. 
  In the second case, by using $p = {2\Delta^*}/{\nu}$, $\LT{\y'}{\y} \le \gamma\norm{\y' - \y}$ for any $\y'\in\conv(\Ycal)$ (\cref{assump:nu}), and the triangle inequality, we obtain
  \begin{align}
    \E[\LT{\psi_\Omega(\thb)}{\y}]
    \le{}& \prn*{1 - 2\Delta^*/\nu} \gamma\norm{\ynep - \y} + (2\Delta^*/\nu) \gamma\norm{\yhat(\thb) - \y} \\
    \le{}& \prn*{1 - 2\Delta^*/\nu} \gamma(\norm{\ynep - \yhat(\thb)} + \norm{\yhat(\thb) - \y}) + (2\Delta^*/\nu) \gamma\norm{\yhat(\thb) - \y} \\
    ={}& \prn*{1 - 2\Delta^*/\nu}\gamma\Delta^* + \gamma\Delta. 
  \end{align}
  Hence, it suffices to prove $\prn*{1 - 2\Delta^*/\nu}\gamma\Delta^* + \gamma\Delta \le 2\gamma\Delta^2/\nu$; by dividing both sides by $\gamma\nu$ and letting $u = \Delta^*/\nu$ and $v = \Delta/\nu$, this can be simplified as $2u^2 + 2v^2 - u - v \ge 0$. 
  From the triangle inequality and $\ynep \neq \y$, we have $\Delta^* + \Delta \ge \norm{\ynep - \y} \ge \nu$, i.e., $u + v \ge 1$.  
  Also, $\Delta^* < \nu/2$ implies $u < 1/2$. 
  Combining them yields $0 \le u < 1/2 < v$. 
  These imply the desired inequality as follows:
  \[
    2u^2 + 2v^2 - u - v = (u + v - 1)(2u + 2v - 1) + (2v - 1)(1 - 2u) \ge 0.
  \]
  Therefore, we have $\E[\LT{\psi_\Omega(\thb)}{\y}] \le 2\gamma\Delta^2/\nu$ in any case, completing the proof.
\end{proof}

As we will see in \cref{prop:surrogate_gap}, given a possibly randomized decoding function $\psi:\R^d\to\Ycal$, a sufficient condition for achieving finite surrogate regret bounds is the existence of $a \in (0, 1)$ such that $\E[\LT{\psi(\thb)}{\y}] \le (1 - a) \LSOmega{\thb}{\y}$ holds for any $(\thb, \y) \in \R^d\times\Ycal$, which leads to a surrogate regret bound proportional to $1/a$. 
We call the quantity $a\LSOmega{\thb}{\y}$ the \emph{surrogate gap}.\footnote{The original definition of the surrogate gap \Citep{Van_der_Hoeven2020-ug} slightly differs, but represents a similar quantity.} 
\Cref{lem:expected_target_bound} will ensure that our randomized decoding offers meaningful surrogate gaps.

\paragraph{Necessity of mixing $\ynep$ and $\yrnd$.}
Our randomized decoding is a mixture of two strategies: returning $\ynep$ or random $\yrnd$.\footnote{While \Citet{Van_der_Hoeven2020-ug} uses a similar mixing strategy, a difference lies in the definition of $\yrnd$, which is crucial for shaving the $O(d)$ factor in the case of online multiclass classification. See \cref{section:multiclass} for a detailed discussion.} 
We explain that either strategy alone does not yield the desired surrogate gap.
Let us discuss the deterministic decoding that always returns $\ynep$.
Consider binary classification with $\Ycal = \set{\e_1, \e_2}$.
Let $\y = \e_1$ be the ground truth and $\thb = (\theta_1, \theta_2) = (1, 1+\ln(2^{1+\varepsilon} - 1))$ for some small $\varepsilon > 0$, which slightly favors $\e_2$ by mistake. 
Then, the logistic loss is $\log_2(1 + \exp(\theta_2 - \theta_1)) = 1+\varepsilon$, and the 0-1 loss is 1 since the deterministic decoding converts $\thb$ to $\e_2$.  
Thus, only a surrogate gap with $a \le \frac{\varepsilon}{1+\varepsilon}$ is left, leading to an arbitrarily large $\Omega(1/\varepsilon)$ surrogate regret bound. 
By contrast, our randomized decoding applied to this setting yields a surrogate gap with $a = 1 - \ln2 \in (0, 1)$ (see \cref{thm:expected_regret_multi}).
Further investigation of this example also suggests that the multiplicative constant of~$4$ in \cref{lem:expected_target_bound} cannot be smaller than $2/\ln2 \approx 2.89$ (see \cref{asec:tightness_of_lemma}).
Next, we discuss the strategy that always returns random $\yrnd$.
If we do so (i.e., fix $p$ to $1$ in the proof of \cref{lem:expected_target_bound}), we only have 
$\E[\LT{\psi_\Omega(\thb)}{\y}] = \LT{\yhat(\thb)}{\y} \le \gamma\Delta$ by \cref{assump:nu} and $\lambda\Delta^2/2 \le \LSOmega{\thb}{\y}$ by \cref{prop:fyloss_properties}.
These do not imply the desired relation, $\E[\LT{\psi_\Omega(\thb)}{\y}] \le (1 - a)\LSOmega{\thb}{\y}$, when $\Delta \ll 1$.
(While we have $\E[\LT{\psi_\Omega(\thb)}{\y}] \lesssim \sqrt{\LSOmega{\thb}{\y}}$, this does not enable us to exploit the surrogate gap; see the proof of \cref{prop:surrogate_gap}.)
By adjusting the bias toward $\y^*$, we can avoid this issue when $\yhat(\thb)$ is very close to some $\y^*$ while moderating the penalty of mistake, $\y^* \neq \y$.\looseness=-1

\subsection{Implementation of Randomized Decoding}\label{subsec:implementation}
\Cref{alg:decoding} involves computing $\yhat(\thb)$ and $\ynep$, and sampling $\yrnd$. 
We can obtain $\yhat(\thb)$ by solving the convex optimization in \cref{step:regularized_max}, and efficient methods for this problem are extensively discussed in \citet[Section~8.3]{Blondel2020-tu}; also, we can use a fast Frank--Wolfe-type algorithm of \citet{Garber2021-tg} to obtain $\yhat(\thb)$, as described shortly. 
Below, we focus on how to obtain $\ynep$ and $\yrnd$ first.

In \cref{step:nearest_extreme_point}, we need to find a nearest extreme point $\ynep \in \Ycal$ to $\yhat(\thb)$ with respect to the distance induced by $\norm{\cdot}$.
In the case of multiclass classification, we can easily do this by choosing $i \in [d]$ corresponding to the largest entry in $\yhat(\thb)$ and setting $\ynep = \e_i$.
More generally, \cref{prop:nep} ensures that if $\Ycal \subseteq \set{0, 1}^d$, which is a common scenario where $\conv(\Ycal)$ constitutes a 0-1 polytope, and $\norm{\cdot}$ is an $\ell_p$-norm, we can find such a nearest extreme point by solving a linear optimization problem.\looseness=-1 
\begin{proposition}\label[proposition]{prop:nep}
  Let $\Ycal \subseteq \set*{0, 1}^d$ and $p \in [1,+\infty)$. 
  For any $\y' \in \conv(\Ycal)$, we can find a nearest extreme point $\ynep \in \Ycal$ to $\y'$ with respect to $\norm{\cdot}_p$, i.e., $\ynep \in \argmin\Set*{\norm{\y - \y'}_p}{\y \in \Ycal}$, via a single call to a linear optimization oracle that, for any $\bm{c} \in \R^d$, returns a point in $\argmin\Set*{\inpr{\bm{c}, \y}}{\y \in \Ycal}$. 
\end{proposition}
\begin{proof}
  We can find a nearest point by minimizing $\norm{\y - \y'}_p^p = \sum_{i=1}^d |y_i - y'_i|^p$ over $\y\in \Ycal$. 
  Since we have $y_i \in \set*{0, 1}$, we can rewrite each term as $|1 - y'_i|^py_i + |y'_i|^p(1 - y_i) = (|1 - y'_i|^p - |y'_i|^p)y_i + |y'_i|^p$.
  Therefore, the problem is equivalent to $\min\Set*{\sum_{i=1}^d (|1 - y'_i|^p - |y'_i|^p)y_i}{\y \in \Ycal}$, which we can solve with the linear optimization oracle.
\end{proof}
\Cref{prop:nep} enables efficient computation of $\ynep$ for various structures of $\conv(\Ycal)$: the 0-1 hypercube (multilabel classification), the Birkhoff polytope (ranking), and a general matroid polytope.
\Citet[Section~1.2]{Garber2021-tg} shows more examples where we can compute nearest extreme points.
The 0-1 polytope case is given there only for the $\ell_2$-norm, and \Cref{prop:nep} extends it to $\ell_p$-norms.\looseness=-1

We turn to how to sample $\yrnd \in \Ycal$ such that $\E[\yrnd \,|\, Z=1] = \yhat(\thb)$ in \Cref{step:random_select}. 
This is also easy in multiclass classification: we sample $i \in [d]$ with probability proportional to the $i$th entry of $\yhat(\thb)$ and set $\yrnd = \e_i$. 
In general, if we have a convex combination of extreme points of $\conv(\Ycal)$ that equals $\yhat(\thb)$, we can sample $\yrnd$ by choosing an extreme point with a probability of the corresponding combination coefficient. 
Such a convex combination can be obtained by applying a Frank--Wolfe-type algorithm to $\min\Set*{\norm{\y - \yhat(\thb)}_2^2}{\y \in \conv(\Ycal)}$, as considered in \citet{Combettes2023-cv} (or, we may directly compute $\yhat(\thb)$ with a Frank--Wolfe-type algorithm). 
In particular, given that we can efficiently compute nearest extreme points as discussed above, we can use the linearly convergent Frank--Wolfe algorithm of \citet[Theorem~5]{Garber2021-tg}, which returns an $\varepsilon$-approximation of $\yhat(\thb)$ as a convex combination of only $O(M\ln(d/\varepsilon))$ extreme points (typically, $M = O(d^2)$).

\section{Surrogate Regret Bounds for Online Structured Prediction}\label{sec:regret_bounds}

\begin{algorithm}[tb]
  \caption{Learning procedure for online structured prediction}
  \label[algorithm]{alg:sgaptron}
  \begin{algorithmic}[1]
    \REQUIRE{$\mathsf{Alg}$ with domain $\Wcal$ and decoding function $\psi_\Omega$ (\cref{alg:decoding})}
    \STATE Set $\W_1$ to the all-zero matrix
    \FOR{$t=1,\dots,T$}
      \STATE Receive $\x_t$ and compute $\thb_t = \W_t\x_t$
      \STATE Play $\yprd_t = \psi_\Omega(\thb_t)$ and observe $\y_t$
      \STATE Send $\S_t$ (or ($\x_t, \y_t$)) to $\mathsf{Alg}$ and get $\W_{t+1}$ in return
    \ENDFOR
  \end{algorithmic}
\end{algorithm}

We analyze the surrogate regret for online structured prediction.
To simplify the notation, let $L_t(\y) = \LT{\y}{\y_t}$ and $\S_t(\W) = \LSOmega{\W\x_t}{\y_t}$ be the target and surrogate losses in the $t$th round, respectively.
We also use $\E_t$ to represent the expectation taken only with respect to the randomness of the randomized decoding to produce $\yprd_t$ in the $t$th round, i.e., $\E_t$ is conditioned on $\yprd_1,\dots,\yprd_{t-1}$. 
The learning procedure is summarized in \Cref{alg:sgaptron}.
In each $t$th round, the learner receives $\x_t$, computes $\thb_t = \W_t\x_t$, plays $\yprd_t$ obtained by decoding $\thb_t$, and observe $\y_t$. 
The learner updates $\W_t$ using an online convex optimization algorithm, denoted by $\mathsf{Alg}$, with domain $\Wcal$ and loss function $\S_t$.
Below, we assume that $\Wcal$ contains the all-zero matrix and set $\W_1$ to it for convenience.

As with \Citet{Van_der_Hoeven2020-ug}, we here use the online gradient descent (OGD) with a constant learning rate $\eta>0$ as $\mathsf{Alg}$; 
we discuss using other online learning methods in \cref{asec:other-oco-algorithms}.
This OGD achieves the following regret bound for any $\U \in \Wcal$ \citep[Theorem~2.13]{orabona2023modern}:
\begin{equation}\label{eq:ogd-constant-regret}
  \sum_{t=1}^T \prn*{
  \S_t(\W_t) - \S_t(\U)
  }
  \le 
  \frac{\norm{\U}_\mathrm{F}^2}{2\eta} + \frac{\eta}{2}\sum_{t=1}^T \norm{\nabla\S_t(\W_t)}_\mathrm{F}^2.
\end{equation}
The next proposition highlights how to obtain a template of finite surrogate regret bounds by learning $\W_t$ with this OGD and exploiting the surrogate gap.
\begin{proposition}\label[proposition]{prop:surrogate_gap}
  Assume that there exist constants $a \in (0, 1)$ and $b > 0$ satisfying the following conditions for $t = 1,\dots,T$:
  (i) $\E_t[L_t(\yprd_t)] \le (1 - a) \S_t(\W_t)$ and (ii) $\norm{\nabla\S_t(\W_t)}_\mathrm{F}^2 \le b\S_t(\W_t)$.\footnote{Surrogate losses satisfying this inequality are said to be \emph{regular} in \Citet{Van_der_Hoeven2021-wi}.}
  Let $\mathsf{Alg}$ be OGD with learning rate $\eta = \frac{2}{b}\min\set*{\frac12, a}$.
  Then, it holds that
  \begin{equation}
    \sum_{t=1}^T \E_t[L_t(\yprd_t)] \le \sum_{t=1}^T \S_t(\U) + 
    \underbrace{
      \frac{(1 - a)b\norm{\U}_\mathrm{F}^2}{4\prn*{1 - \min\set*{\frac12, a}}\min\set*{\frac12, a}}
    }_{\mathrm{(\hypertarget{markA}{A})}}.
  \end{equation}
\end{proposition}
\begin{proof}
  By substituting (ii) into \eqref{eq:ogd-constant-regret}, we have
  $\sum_{t=1}^T \prn*{
    \S_t(\W_t) - \S_t(\U)
    }
  \le 
  \frac{\norm{\U}_\mathrm{F}^2}{2\eta} + \frac{\eta b}{2}\sum_{t=1}^T \S_t(\W_t)$. 
  Noting $\frac{\eta b}{2} < 1$, we rearrange the terms to obtain a so-called \textit{$L^\star$ bound} \citep[Section~4.2.3]{orabona2023modern}, whose right-hand side depends on $\sum_{t=1}^T \S_t(\U)$ as follows:
  \begin{equation}
    \sum_{t=1}^T \prn*{
    \S_t(\W_t) - \S_t(\U)
    }
    \le 
    \prn*{1 - \frac{\eta b}{2}}^{-1}
    \prn*{
    \frac{\norm{\U}_\mathrm{F}^2}{2\eta} + \frac{\eta b}{2}\sum_{t=1}^T \S_t(\U)
    }.
  \end{equation}
  Combining this with (i) implies that the surrogate regret, $\sum_{t=1}^T\E_t\brc*{L_t(\yprd_t)} - \sum_{t=1}^T \S_t(\U)$, is at most
  \begin{equation}
    \begin{aligned}
      &(1-a)\sum_{t=1}^T \prn*{
        \S_t(\W_t) - \S_t(\U)
      }
      - a\sum_{t=1}^T \S_t(\U)
      \\
      \le{}&      
      (1-a)\prn*{1 - \frac{\eta b}{2}}^{-1}\frac{\norm{\U}_\mathrm{F}^2}{2\eta}
      -
      \prn*{
        a
        -
        (1-a)\prn*{1 - \frac{\eta b}{2}}^{-1}
        \frac{\eta b}{2}
      }      
      \sum_{t=1}^T \S_t(\U).
      \end{aligned}
  \end{equation}
  Since $\smash{\frac{\eta b}{2} \le a}$ implies $a - (1-a)\prn*{1 - \frac{\eta b}{2}}^{-1}
  \frac{\eta b}{2} \ge 0$, ignoring the second term does not decrease the right-hand side. 
  Substituting $\eta = \frac{2}{b}\min\set*{\frac12, a}$ into the first term yields the desired bound.  
\end{proof}
The last part in the proof highlights the fundamental idea for achieving a finite surrogate regret bound: offsetting the increase in the regret of OGD, which originates from $\frac{\eta}{2}\sum_{t=1}^T \norm{\nabla\S_t(\W_t)}_\mathrm{F}^2$ in~\eqref{eq:ogd-constant-regret}, with the cumulative surrogate gap, $a \sum_{t=1}^T \S_t(\U)$, by setting $\eta$ to a sufficiently small value.
This is based on the original idea of exploiting the surrogate gap by \Citet{Van_der_Hoeven2020-ug}. 
The crux of this fundamental idea lies in conditions~(i) and~(ii) in \cref{prop:surrogate_gap}, which we will verify by using \cref{lem:expected_target_bound} and the properties of the Fenchel--Young loss in \cref{prop:fyloss_properties}, respectively.
Consequently, we obtain the following finite surrogate regret bound in expectation for online structured prediction with Fenchel--Young losses, which is the main result of this paper.
\begin{theorem}\label{thm:expected_regret_general}
  Let  
  $\psi_\Omega$ be the randomized decoding given in \cref{alg:decoding} and $C > 0$ a constant with $\max_{t\in[T]}\norm{\x_t}_2 \le C$. 
  If $\lambda > \frac{4\gamma}{\nu}$ holds and $\mathsf{Alg}$ is OGD with learning rate $\eta = \frac{\lambda}{C^2\kappa^2}\min\set*{\frac12, 1-\frac{4\gamma}{\lambda\nu}}$, for any $\U \in \Wcal$, it holds that 
  \begin{align}
  \sum_{t=1}^T \E_t[L_t(\yprd_t)] 
  \le 
  \sum_{t=1}^T \S_t(\U) + 
  \frac{2\gamma C^2\kappa^2 \norm{\U}_\mathrm{F}^2}{
    \lambda^2\nu
    \prn*{
      1 - \min\set*{\frac12, 1-\frac{4\gamma}{\lambda\nu}}
    }
    \min\set*{\frac12, 1-\frac{4\gamma}{\lambda\nu}}
  }.
  \end{align}
\end{theorem}
\begin{proof}
  Since \cref{lem:expected_target_bound} implies $\E_t[L_t(\yprd_t)] \le \frac{4\gamma}{\lambda\nu}\S_t(\W_t)$, condition (i) in \cref{prop:surrogate_gap} holds with $a = 1 - \frac{4\gamma}{\lambda\nu} \in (0, 1)$. 
  Furthermore, \cref{prop:fyloss_properties} implies $\nabla\S_t(\W_t) = (\yhat(\thb_t) - \y_t) \x_t^\top$ and $\norm{\yhat(\thb_t) - \y_t}^2 \le \frac{2}{\lambda}\S_t(\W_t)$; combining them with $\norm{\y\x^\top}_\mathrm{F}^2 = \tr(\x^\top\x\y^\top\y) = \norm{\x}_2^2\norm{\y}_2^2$, $\norm{\x_t}_2 \le C$, and $\norm{\cdot}_2 \le \kappa\norm{\cdot}$ yields\looseness=-1
  \[
    \norm{\nabla\S_t(\W_t)}_\mathrm{F}^2 = \norm{\yhat(\thb_t) - \y_t}_2^2\norm{\x_t}_2^2 \le \CC\kappa^2\norm{\yhat(\thb_t) - \y_t}^2 \le \frac{2\CC\kappa^2}{\lambda}\S_t(\W_t).
  \]
  Thus, condition~(ii) with $b = \frac{2\CC\kappa^2}{\lambda}$ holds. 
  Therefore, \cref{prop:surrogate_gap} provides the desired bound.
\end{proof}
It is also worth mentioning that the above OGD is \textit{parameter-free} in the sense that the learning rate $\eta$ is tuned without any knowledge of $\U$ or the size of domain $\Wcal$ (cf.\ \citet{Mcmahan2012-tf}, \citet{Orabona2013-ng}, and \citet{Cutkosky2018-bf}).\footnote{
  The line of work on parameter-free learning achieves regret bounds that depend almost linearly on the comparator's norm via non-trivial techniques. 
  Compared to this, achieving the surrogate regret bound depending on $\norm{\U}_\mathrm{F}^2$ is easier.
  }
However, the constant learning rate may result in poor empirical performance, particularly when $a = 1 - \frac{4\gamma}{\lambda\nu}$ is very small.
\Cref{asec:parameter_free} shows that we can alternatively use a parameter-free algorithm of \citet{Cutkosky2018-bf} to achieve a finite surrogate regret bound.

\paragraph{High-probability bound.}
Similar to \Citet{Van_der_Hoeven2021-wi}, we can obtain a finite surrogate regret bound that holds with high probability.
Define random variables $Z_t \coloneqq L_t(\yprd_t) - \E_t[L_t(\yprd_t)]$ for $t = 1,\dots,T$, where the randomness comes from the randomized decoding. 
A crucial step is to ensure that the cumulative deviation $\sum_{t=1}^T Z_t$ grows only at the rate of $\sqrt{\sum_{t=1}^T \S_t(\U)}$, in which \cref{lem:expected_target_bound} again turns out to be helpful. 
Once it is shown, we can obtain a high-probability bound by offsetting the regret of OGD, plus $\sum_{t=1}^T Z_t$, with $a\sum_{t=1}^T \S_t(\U)$.
See \cref{asec:high_probability_regret_general} for the proof.
\begin{restatable}{theorem}{highprobability}\label{thm:high_probability_regret_general}
  Assume the same condition as \cref{thm:expected_regret_general} except for the learning rate of OGD, which we here set as $\eta = \frac{a}{b}$.
  Let $D$ be the diameter of $\conv(\Ycal)$ in terms of $\norm{\cdot}$ and $\delta \in (0, 1)$. 
  Then, with probability at least $1 - \delta$, for any $\U\in\Wcal$, it holds that 
  \begin{equation}
    \sum_{t=1}^T L_t(\yprd_t) \le 
    \sum_{t=1}^T \S_t(\U) +  
    \prn*{1 - \frac{4\gamma}{\lambda\nu}}^{-1}
    \prn*{
      \frac{8\gamma C^2\kappa^2\norm{\U}_\mathrm{F}^2}{\lambda^2\nu} + \gamma D\ln\frac{1}{\delta}
    }.
  \end{equation}
\end{restatable}

\begin{remark}[The case of adaptive adversary]
\Cref{thm:expected_regret_general} remains true even against an adaptive adversary since \cref{lem:expected_target_bound} ensures that condition~(i) in \cref{prop:surrogate_gap} holds for any adaptive sequence $(\x_1, \y_1),\dots,(\x_T, \y_T)$ and the regret bound of OGD in \eqref{eq:ogd-constant-regret} applies to the adaptive case as well.
The high-probability bound in \cref{thm:high_probability_regret_general} also remains valid in the adaptive case since an additional concentration argument used in the proof is irrelevant to the adversary's type. 
\end{remark}

\begin{remark}[Asymptotic behavior when $a \to 1$]
The surrogate regret bound (\hyperlink{markA}{A}) in \cref{prop:surrogate_gap} simplifies to $\frac{b\norm{\U}_\mathrm{F}^2}{4a}$ if $a \le 1/2$ and to $(1 - a)b\norm{\U}_\mathrm{F}^2$ if $a > 1/2$, where the latter expression is smaller when $a > 1/2$.
Notably, the bound vanishes when $a \to 1$. 
This property has not been observed in previous studies (\Citealp{Van_der_Hoeven2020-ug}; \Citealp{Van_der_Hoeven2021-wi}), and we have obtained this by taking advantage of the $L^\star$ bound. 
Note that (i) in \cref{prop:surrogate_gap} implies that $\E_t[L_t(\yprd_t)]/\S_t(\W_t)$ goes to zero when $a \to 1$.
Therefore, this asymptotic behavior reflects a rationale that the surrogate regret bound should vanish when the target loss scales down relative to the surrogate loss.
As in the proof of \cref{thm:expected_regret_general}, $1 - a$ and $b$ are proportional to $1/\lambda$, and hence the surrogate regret bound in \cref{thm:expected_regret_general} vanishes at the rate of $1/\lambda^2$. 
The high-probability surrogate regret bound in \cref{thm:high_probability_regret_general} also decreases at the rate of $1/\lambda^2$, while the $\gamma D\ln\frac{1}{\delta}$ term persists as it comes from the randomness of the decoding.
Here, we can increase or decrease $\lambda$ by scaling up or down the regularization function $\Omega$ generating the Fenchel--Young loss (see also \cref{subsec:application}), although increasing $\lambda$ generally leads to larger $\S_t(\U)$.\looseness=-1
\end{remark}

\subsection{Application to Specific Problems}\label{subsec:application}
\cref{thm:expected_regret_general,thm:high_probability_regret_general} provide finite surrogate regret bounds for various online structured prediction problems that satisfy \cref{assump:nu} if $\lambda > \frac{4\gamma}{\nu}$ (or $a = 1 - \frac{4\gamma}{\lambda\nu} > 0$) holds, which requires $\Omega$ generating the Fenchel--Young loss to be sufficiently strongly convex. 
Notably, this requirement is automatically satisfied in the case of multiclass classification with the logistic loss (see \cref{section:multiclass}). 
In the multilabel classification example in \cref{subsec:example}, we have $\nu = 1$, $\gamma = \frac{1}{\sqrt{d}}$, and $\lambda = 1$; therefore, $\lambda > \frac{4\gamma}{\nu}$ holds if $d > 16$.\footnote{Note that the target loss is scaled to $[0, 1]$. If not scaled, we need to scale up $\Psi$ to satisfy $\lambda > 4\sqrt{d}$.}
In general, we can take advantage of the Fenchel--Young loss framework to satisfy $\lambda > \frac{4\gamma}{\nu}$: 
since we may use any function $\Omega = \Psi + I_{\conv(\Ycal)}$ to generate a Fenchel--Young loss, we can scale up $\Psi$ to satisfy $\lambda > \frac{4\gamma}{\nu}$ if necessary. 
In the ranking example in \cref{subsec:example}, we have $\nu = 4$, $\gamma = \frac{1}{2n}$, and $\lambda = \frac{1}{n\mu}$, where $\mu > 0$ controls the scale of $\Psi = -\frac{1}{\mu}\Hs$. 
Thus, $\lambda > \frac{4\gamma}{\nu}$ holds if $\mu < 2$. 
Note that the dependence of $\lambda$ on $\gamma$ is inevitable because the surrogate loss encodes no information about the target loss per se. 
Nonetheless, we can scale $\lambda$ to satisfy $\lambda > \frac{4\gamma}{\nu}$ as $\nu$ is often lower bounded: $\nu \ge 1$ holds if $\Ycal \subseteq \Z^d$ and $\norm{\cdot}$ is an $\ell_p$-norm, and $\nu \ge 2^{1/p}$ if $\y^\top\ones$ is constant.

\subsection{Online-to-Batch Conversion}\label{subsec:onlinetobatch}
We discuss converting surrogate regret bounds to guarantees for offline structured prediction.
In general, surrogate regret bounds may not admit online-to-batch conversion because we cannot apply Jensen's inequality to non-convex target loss. 
In our case, \cref{lem:expected_target_bound}, which bounds target loss by convex surrogate loss, enables us to sidestep this issue, leading to the following result.
\begin{restatable}{theorem}{onlinetobatch}\label{thm:online-to-batch}
  Assume the same condition as \cref{thm:expected_regret_general}. 
  If $(\x_1, \y_1),\dots,(\x_T, \y_T)$ are drawn i.i.d.\ from an underlying joint distribution on $\Xcal$ and $\Ycal$, for any $\U\in\Wcal$, it holds that 
  \begin{align}
    \E[\LT{\psi_\Omega(\overline{\W}\x)}{\y}] 
    \le 
    \E[\LSOmega{\U\x}{\y}] + 
    \frac{1}{T}\cdot
    \frac{2\gamma C^2\kappa^2 \norm{\U}_\mathrm{F}^2}{
      \lambda^2\nu
      \prn*{
        1 - \min\set*{\frac12, 1-\frac{4\gamma}{\lambda\nu}}
      }
      \min\set*{\frac12, 1-\frac{4\gamma}{\lambda\nu}}
    },
  \end{align}  
  where $\overline{\W} = \frac{1}{T}\sum_{t=1}^T\W_t$ is the average of outputs of $\mathsf{Alg}$.
\end{restatable}
\noindent
The above bound differs from common excess risk bounds studied in statistical learning, and hence comparing them directly is difficult. 
However, since (super) fast convergence for structured prediction has been already established by \citet{Cabannes2021-vv}, we do not think the bound in \cref{thm:online-to-batch} itself is of particular significance. 
Nonetheless, \cref{thm:online-to-batch} offers a better understanding of the relationship between online learning guarantees based on the surrogate regret and statistical learning theory with margin conditions.
We present the proof of \cref{thm:online-to-batch} in \cref{asec:online-to-batch-proof} and discuss a connection to excess risk bounds in \cref{asec:comparison_offline}.

\section{Improved Surrogate Regret for Online Multiclass Classification with Logistic Loss}\label{section:multiclass}
We present an $O(\norm{\U}_\mathrm{F}^2)$ surrogate regret bound for online multiclass classification with the logistic loss by using our general result for structured prediction, thereby improving the $O(d\norm{\U}_\mathrm{F}^2)$ bound of \Citet{Van_der_Hoeven2020-ug}. 
In this section, we let $\Ycal = \set{\e_1,\dots,\e_d}$ and $\norm{\cdot} = \norm{\cdot}_1$. 
We have $\kappa = 1$ since $\norm{\cdot}_1 \ge \norm{\cdot}_2$. 
The target loss is the 0-1 loss, $\LT{\y'}{\e_i} = \ind_{\y' \neq \e_i}$. 
Note that we have $\nu = 2$ and $\gamma = 1/2$, as explained in \cref{subsec:example}. 
We use the same logistic loss as that used by \Citet{Van_der_Hoeven2020-ug}. 
Specifically, for any $\thb \in \R^d$ and $\e_i \in \Ycal$, we define the logistic loss as 
\[
  \LSlog{\thb}{\e_i} \coloneqq -\log_2 \sigma_i(\thb),
\]
where $\sigma_i(\thb) \coloneqq \frac{\exp(\theta_i)}{\sum_{j=1}^d \exp(\theta_j)}$ is the softmax function.
This logistic loss is expressed as a Fenchel--Young loss up to a constant factor. 
For any $\y \in \triangle^d$, let $\Omega$ be an entropic regularizer given by
\begin{equation}\label{eq:entropy-omega}
  \Omega(\y) = -\Hs(\y) + I_{\triangle^d}(\y).
\end{equation}
The Fenchel--Young loss generated by this $\Omega$ is $\LSOmega{\thb}{\e_i} = -\ln \sigma_i(\thb)$ (see \citet{Blondel2020-tu}), hence $\LSlog{\thb}{\e_i} = \frac{1}{\ln2} \LSOmega{\thb}{\e_i}$. 
Moreover, $\yhat(\thb)$ equals $(\sigma_1(\thb),\dots,\sigma_d(\thb))^\top$, which we can efficiently compute in the randomized decoding (\cref{alg:decoding}) without iterative optimization methods.\looseness=-1 

By applying \cref{alg:sgaptron} to the above setting, we obtain the following surrogate regret bound.

\begin{theorem}\label{thm:expected_regret_multi}
  Let $\psi_\Omega$ be the randomized decoding given in \cref{alg:decoding} with the entropic regularizer $\Omega$ in \eqref{eq:entropy-omega} and $C > 0$ a constant such that $\max_{t\in[T]}\norm{\x_t}_2 \le C$. 
  If we apply OGD with $\eta = \frac{(1 - \ln 2)\ln 2}{\CC}$ to loss functions $\S_t(\W) = \LSlog{\W\x_t}{\y_t}$ ($t=1,\dots,T$), for any $\U \in \Wcal$, \cref{alg:sgaptron} achieves
  \begin{align}
  \sum_{t=1}^T \E_t[\ind_{\yprd_t \neq \y_t}] 
  \le 
  \sum_{t=1}^T \S_t(\U) + 
  \frac{\CC\norm{\U}_\mathrm{F}^2}{2(1 - \ln2)\ln2}.
  \end{align}
\end{theorem}

\begin{proof}
  The proof resembles that of \cref{thm:expected_regret_general}. 
  From Pinsker's inequality, $-\Hs$ is $1$-strongly convex w.r.t.\ $\norm{\cdot}_1$ over $\triangle^d$ (i.e., $\lambda = 1$). 
  Thus, $\E[\LT{\psi_\Omega(\thb)}{\y}] \le \frac{4\gamma}{\lambda\nu} \LSOmega{\thb}{\y} = \frac{4\cdot1/2}{1\cdot2} \ln2 \cdot \LSlog{\thb}{\y}$ holds due to \cref{lem:expected_target_bound}, leaving a surrogate gap with $a = 1 - \ln2 \in (0, 1/2)$ in \cref{prop:surrogate_gap}. 
  We also have $\norm{\nabla\S_t(\W_t)}_\mathrm{F}^2 \le \frac{2\CC}{\ln2} \S_t(\W_t)$ due to \Citet[Lemma~2]{Van_der_Hoeven2020-ug}, i.e., $b = \frac{2\CC}{\ln2}$.
  By setting $\eta = \frac{2a}{b} = \frac{(1 - \ln 2)\ln 2}{\CC}$, \cref{prop:surrogate_gap} implies the desired bound of $\frac{b\norm{\U}_\mathrm{F}^2}{4a} = \frac{\CC\norm{\U}_\mathrm{F}^2}{2(1 - \ln2)\ln2}$.  
\end{proof}

\paragraph{Difference from \Citet{Van_der_Hoeven2020-ug} and \Citet{Van_der_Hoeven2021-wi}.} 
The main technical difference from the previous studies lies in how to decode $\thb \in \R^d$ to $\y \in \Ycal$.
Specifically, when $\thb$ is not confident about any of $d$ classes, their methods increase the likelihood of choosing a class uniformly at random (i.e., uniform exploration with probability $\frac{1}{d}$), which yields a surrogate gap with $a = \frac{1}{d}$ in \cref{prop:surrogate_gap} (see \Citet[Lemma~1]{Van_der_Hoeven2021-wi}), resulting in the extra $d$ factor. 
Our randomized decoding instead returns random $\yrnd \in \Ycal$ with $\E[\yrnd \,|\, Z=1] = \yhat(\thb)$, which exploits $\thb$ more aggressively and yields a surrogate gap with $a = 1 - \ln2$, thus achieving the improved bound of $O(\norm{\U}_\mathrm{F}^2)$.\footnote{The dependence on $C$ is identical in our bound and that of \Citet{Van_der_Hoeven2020-ug}.}
Apart from this, the two decoding procedures have distinct pros and cons: 
uniform exploration is often extensive for structured spaces,\footnote{For example, if $\Ycal$ consists of perfect matchings of a (possibly incomplete) bipartite graph with $n$ vertices, the current fastest \emph{fully polynomial almost uniform sampler} takes $O(n^7 \ln n)$ time~\citep{Jerrum2004-cm,Bezakova2008-vc}, which is known to be impractical \citep{Newman2020-el}.} 
while our randomized decoding enjoys efficient implementations given linear optimization oracles on $\Ycal$, as discussed in \cref{subsec:implementation}. 
On the other hand, uniform exploration is applicable to the bandit setting and works with the (smooth) hinge loss. 
Investigating how to apply our randomized decoding to the bandit setting and how the resulting bound compares with the state-of-the-art \Citep{Van_der_Hoeven2020-ug,Agarwal2022-hk} will be an interesting future direction, whereas applying the randomized decoding to the (smooth) hinge loss seems difficult, as discussed in \cref{footnote:cost-sensitive}.

Finally, we give a lower bound showing that \cref{thm:expected_regret_multi} is asymptotically tight up to $\ln d$ factors if $\norm{\U}_\mathrm{F} = \Theta(B)$ holds for the $\ell_2$-diameter $B$ of the domain $\Wcal$.
The proof, deferred to \cref{asec:lower_bound}, is inspired by that of \Citet{Van_der_Hoeven2021-wi} for obtaining an $\Omega(dB^2)$ lower bound for the smooth hinge loss, which we modify to deal with the logistic loss.

\begin{restatable}{theorem}{lowerbound}\label{thm:lower_bound}
  Let $d \ge 2$. 
  For $B = \Omega(\ln(dT))$, there exists a sequence $(\x_1,\y_1),\dots,(\x_T,\y_T)$ such that $\norm{\x_t}_2 = 1$ for $t=1,\dots,T$ and any possibly randomized algorithm incurs an $\Omega(B^2/\ln^2 d)$ surrogate regret with respect to the logistic surrogate loss.
\end{restatable}

\paragraph{Logistic vs.\ smooth hinge.}
  As mentioned in \cref{sec:intro}, larger surrogate loss makes it easier to bound the surrogate regret. 
  Thus, considering the (nearly) tight bounds of $O(\norm{\U}_\mathrm{F}^2)$ and $O(d\norm{\U}_\mathrm{F}^2)$ for the logistic and smooth hinge losses, respectively, one may think that the logistic loss is always larger than the smooth hinge loss. 
  However, this is not the case. 
  For example, consider a binary classification setting with $\Ycal = \set*{\e_1, \e_2}$, where $\e_1$ is the ground truth. 
  If an estimator $\U$ yields $(\theta_1, \theta_2)^\top = \U\x_t$ with $\theta_1 - \theta_2 = 0.3$, the logistic and smooth hinge losses take $\log_2(1+\exp(\theta_2 - \theta_1)) \approx 0.8$ and $\max\set{1 - (\theta_1 - \theta_2)^2, 0} \approx 0.9$, respectively, implying that the logistic loss is not always larger than the smooth hinge loss.

\subsection*{Acknowledgements}
We thank anonymous COLT reviewers, who provided truly constructive feedback.
In particular, their review comments helped us improve the surrogate regret bounds so that they vanish when $a \to 1$ and brought our attention to the parameter-freeness of the original method by \Citet{Van_der_Hoeven2020-ug}.
This work was supported by JST ERATO Grant Number JPMJER1903, JST ACT-X Grant Number JPMJAX210E, and JSPS KAKENHI Grant Number JP22K17853.

\newrefcontext[sorting=nyt]
\printbibliography[heading=bibintoc]

\clearpage
\appendix

\DeclareCiteCommand{\citeyear}
    {}
    {\bibhyperref{\printdate}}
    {\multicitedelim}
    {}

\crefalias{section}{appendix} 

\section{Literature Review of Structured Prediction}\label{asec:related_work}
In this literature review, the term ``(surrogate) regret'' refers to the (surrogate) excess risk.

Earlier studies investigated statistical inference on conditional random fields \citep{Lafferty2001}, max-margin models \citep{Bartlett2004}, and spanning trees \citep{Koo2007}.
\Citet{Tsochantaridis2005} is a seminal work to provide a general framework based on loss functions for structured prediction problems by extending support vector machines.
Independent of the development of SparseMAP \citep{Niculae2018-qg}, \citet{Pillutla2018} proposed a tractable algorithm to optimize the structured hinge loss by introducing a smoothed decoder.

With the SELF framework (\citealp{Ciliberto2016-nl}, \citeyear{Ciliberto2020-zr}), regret bounds depend on the spectral norm of $\V$, which may be exponential in the natural dimension of the output space. 
To address this issue, \citet[Appendix~E]{Osokin2017} improved the surrogate regret bound of the quadratic loss for specific target losses such as the block 0-1 and Hamming losses, showing that the dependency of the surrogate regret on the matrix norm can be lifted.
\Citet[Theorem~3.1]{Nowak-Vila2019} systematically extended this result to many multilabel and ranking target losses by obtaining low-rank decomposition of $\V$ in SELF for those losses.
Moreover, \citet{Nowak-Vila2020} established the surrogate regret bounds beyond the quadratic loss, for max-margin surrogate losses generated by the Fenchel--Young losses;
\citet{Nowak-Vila2022} studied necessary conditions of a structured target loss for max-margin losses to be Fisher consistent.
\Citet{Cabannes2020} elucidated that partial label learning, a type of learning problem with ambiguous structured inputs, can be cast into the framework of the regularized least-square decoder and established the surrogate regret bound for the target loss called the infimum loss, encompassed into SELF.
Recently, \citet{Cabannes2021-vv} showed fast convergence rates for the excess risk of SELF under a generalized Tsybakov margin condition; in particular, their result implies exponential convergence in the number of samples under the hard margin condition. 
\Citet{Li2021-pd} studied generalization bounds for surrogate losses, which imply a fast convergence rate when surrogate losses are smooth.

Apart from the aforementioned studies, another stream of research has studied the consistency of structured prediction problems via \emph{polyhedral losses}.
\Citet{Finocchiaro2019} and \citet{Thilagar2022} studied the consistency of classification with abstention, top-$k$ prediction, and Lov\'{a}sz hinge loss;
\citet{Wang2020} studied the consistent target loss of the Weston--Watkins hinge loss.

The literature is scarce when it comes to the online setting (except for online multiclass classification discussed in \cref{subsec:related-work}).
\Citet{Martins2011-ol} studied an online learning approach to structured prediction with multiple kernels based on standard regret bounds for convex surrogate losses.

\section{Note on CRF Loss}\label{asec:crf}
We confirm that the \emph{Conditional Random Field (CRF)} loss \citep{Lafferty2001} satisfies \cref{prop:fyloss_properties}. 
Consequently, we can treat it as a specific Fenchel-Young loss to obtain the results presented in the main text, similar to the logistic and SparseMAP losses. 
Although the following discussion is elementary, we include it for completeness.

Below, $\Ycal$ is equipped with some total order and the components of any $|\Ycal|$-dimensional vector are arranged in the same order.
As detailed in \citet[Section~7.1]{Blondel2020-tu}, the CRF loss is a Fenchel--Young loss generated by $\Omega(\y) = \min\Set*{-\Hs(\bm{p})}{\bm{p} \in \triangle^{|\Ycal|},\ \E_{Y\sim\bm{p}}[Y] = \y}$. 
For any $\thb \in \R^d$ and $\y\in \Ycal$, the resulting CRF loss is expressed as $\LSOmega{\thb}{\y} = -\ln(\exp(\inpr{\thb, \y})/Z(\thb))$, where $Z(\thb) \coloneqq \sum_{\y \in \Ycal}\exp(\inpr{\thb, \y})$ denotes the partition function. 
The regularized prediction is uniquely determined by the marginal inference: 
$\yhat(\thb) = \sum_{\y \in \Ycal}({\exp(\inpr{\thb, \y})}/{Z(\thb)})\y$. 
Thus, the gradient, $\nabla\LSOmega{\thb}{\y} = \yhat(\thb) - \y$, is also unique. 
In what follows, we prove $\LSOmega{\thb}{\y} \ge \frac{\lambda}{2}\norm{\yhat(\thb) - \y}_1^2$ for $\lambda \coloneqq 1/\max\Set*{\norm{\y}_1^2}{\y\in\Ycal}$ to establish \cref{prop:fyloss_properties}. 
If $\Ycal \subseteq \set{0,1}^d$, this $\lambda$ is at least $1/d^2$.

First, let us observe that for any $\thb \in \R^d$, the CRF loss can be seen as a logistic loss defined on $\R^{|\Ycal|}$ by associating each $\y \in \Ycal$ with a score $\inpr{\thb, \y} \in \R$.
Specifically, let $\bm{s}(\thb) \in \R^{|\Ycal|}$ denote a vector whose component corresponding to $\y \in \Ycal$, denoted by $s_{\y}(\thb)$, equals $\inpr{\thb, \y}$. 
Then, we have $\LSOmega{\thb}{\y} = -\ln(\exp(s_{\y}(\thb))/\sum_{\y' \in \Ycal}\exp(s_{\y'}(\thb)))$.
By regarding this as the logistic loss of $\bm{s}(\thb)$, the Pinsker's inequality (or \citet[Proposition~2]{Blondel2019-gd}) implies 
\begin{equation}\label{eq:crf-pinsker}
  \LSOmega{\thb}{\y} \ge \frac12 \norm{\sigma(\bm{s}(\thb)) - \e_{\y}}_1^2,
\end{equation}
where $\sigma:\R^{|\Ycal|}\to\triangle^{|\Ycal|}$ is the softmax function, i.e., $\sigma_{\y}(\bm{s}(\thb)) = \exp(s_{\y}(\thb))/\sum_{\y' \in \Ycal}\exp(s_{\y'}(\thb))$, and $\e_{\y}\in \R^{|\Ycal|}$ is the standard basis vector that has a single one at the component corresponding to $\y$.

Next, we show $\norm{\sigma(\bm{s}(\thb)) - \e_{\y}}_1^2 \ge \lambda\norm{\yhat(\thb) - \y}_1^2$, which combined with \eqref{eq:crf-pinsker} yields the desired inequality. 
To see this, let $A$ be the $d\times|\Ycal|$ matrix with column vectors $\y \in \Ycal$ aligned horizontally in the same total order. 
We have $\yhat(\thb) = A\sigma(\bm{s}(\thb))$ and $\y = A\e_{\y}$, and hence 
\begin{equation}
  \norm{\yhat(\thb) - \y}_1 = \norm{A(\sigma(\bm{s}(\thb)) - \e_{\y})}_1 \le \norm{A}_1\norm{\sigma(\bm{s}(\thb)) - \e_{\y}}_1.
\end{equation}
Here, $\norm{A}_1$ is the operator norm of $A$ between the $\ell_1$-normed spaces, i.e., the maximum $\ell_1$-norm of the columns of $A$. 
Thus, we have $\norm{A}_1 = 1/\sqrt\lambda$, and hence $\norm{\sigma(\bm{s}(\thb)) - \e_{\y}}_1^2 \ge \lambda\norm{\yhat(\thb) - \y}_1^2$. 

We remark that the marginal inference for computing $\yhat(\thb)$ is sometimes intractable, which has motivated the development of SparseMAP \citep{Niculae2018-qg}. 
We refer the reader to \citet[Section~7.3]{Blondel2020-tu} for a discussion on the computational complexity.

\section{Additional Applications}\label{asec:additional_applications}
We discuss additional applications considered in \citet{Blondel2019-gd}. 
For simplicity, we below let $\norm{\cdot}$ be the $\ell_2$-norm and $\S_\Omega$ the SparseMAP loss generated by $\Omega(\y) = \frac{1}{2}\norm{\y}_2^2 + I_{\conv(\Ycal)}(\y)$, as in the multilabel classification example in \cref{subsec:example}; hence, we have $\lambda = 1$.
We will confirm the condition of $\lambda > \frac{4\gamma}{\nu}$ in \cref{thm:expected_regret_general} and discuss the implementation of randomized decoding (\cref{alg:decoding}).

\subsection{Ranking with Permutahedron}
We consider another ranking scenario with different $\Ycal$. 
Let $\Ycal \subseteq \Z^d$ be the set of all points obtained by permuting $(d, d-1, \dots, 1)^\top \in \Z^d$. 
Then, $\conv(\Ycal)$ is the so-called \emph{permutahedron}. 
Since $\y^\top\ones$ is constant for all $\y \in \Ycal$ and $\norm{\cdot}$ is the $\ell_2$-norm, we have $\nu = \sqrt{2}$. 
We consider measuring how predicted $\y' \in \conv(\Ycal)$ is aligned with the true $\y \in \Ycal$ by $\inpr{\y, \y - \y'}$, which takes $0$ if $\y' = \y$ and $M \coloneqq \frac{d(d^2-1)}{6}$ for the least aligned $\y'$.
Based on this idea, we use a target loss defined by $\LT{\y'}{\y} = \frac{1}{M}\inpr{\y, \y - \y'} \in [0, 1]$, which is affine in $\y'$ and satisfies $\LT{\y'}{\y} \le \frac{1}{M}\norm{\y}_2\norm{\y - \y'}_2 = \gamma\norm{\y - \y'}_2$ for $\gamma = \frac{1}{d-1}\sqrt{\frac{2d+1}{6d(d+1)}}$. 
Therefore, $\lambda > \frac{4\gamma}{\nu}$ holds for $d \ge 3$.

We provide an efficient implementation of randomized decoding (\cref{alg:decoding}) for this setting.
To this end, we show that a similar claim to \cref{prop:nep} holds, though $\conv(\Ycal)$ is not a 0-1 polytope. 

\begin{proposition}\label[proposition]{prop:nearest-permutahedron}
  Let $\y'\in\conv(\Ycal)$ and $\pi$ be a permutation on $[d]$ such that $y'_{\pi(1)} \le \dotsb \le y'_{\pi(d)}$. 
  It holds that $\y^* \coloneqq (\pi^{-1}(1), \dots, \pi^{-1}(d))^\top \in \argmin\Set*{\norm{\y - \y'}_p}{\y \in \Ycal}$, where $p \in [1, +\infty]$.
\end{proposition}
\begin{proof}
  Let $p < \infty$; the case of $p = \infty$ follows by taking the limit $p \to \infty$ in what follows. 
  Suppose that there exists a nearest point $\y \in \Ycal$ to $\y'$ such that $y'_i > y'_j$ and $y_i < y_j$ for some $i, j \in [d]$, which implies $y_i - y'_i < y_j - y'_j$.
  Since $y_i - y'_j$ and $y_j - y'_i$ are contained in $[y_i - y'_i, y_j - y'_j]$, we have $\abs{y_i - y'_i}^p + \abs{y_j - y'_j}^p \ge \abs{y_i - y'_j}^p + \abs{y_j - y'_i}^p$ by the convexity of $x \mapsto \abs{x}^p$.
  Thus, letting $\widecheck{\y} \in \Ycal$ be a point obtained from $\y$ by swapping the $i$th and $j$th components, we obtain
  \begin{align}
    \norm{\y - \y'}_p^p
    &=  \abs{y_i - y'_i}^p + \abs{y_j - y'_j}^p + \sum_{k \ne i,j} \abs{y_k - y'_k}^p\\
    &\ge \abs{\widecheck{y}_i - y'_i}^p + \abs{\widecheck{y}_j - y'_j}^p + \sum_{k \ne i,j} \abs{\widecheck{y}_k - y'_k}^p 
    = \norm{\widecheck{\y} - \y'}_p^p,
  \end{align}
  meaning that $\widecheck{\y}$ is also a nearest point to $\y'$.
  Therefore, there must exist a nearest point such that its components have the same order as those of $\y'$, and any such points have the same distance to $\y'$.
  Consequently, $\y^*$ is a nearest point to $\y'$ with respect to $\norm{\cdot}_p$ among all $\y \in \Ycal$.
\end{proof}
\Cref{prop:nearest-permutahedron} means that we can compute a nearest extreme point $\y^* \in \Ycal$ to any given $\y'\in\conv(\Ycal)$ by sorting the components of $\y'$ in $O(d \ln d)$ time.
Armed with this, we can use the fast Frank--Wolfe algorithm of \citet{Garber2021-tg} to compute $\yhat(\thb)$ and a convex combination of extreme points for sampling $\yrnd$ with $\E[\yrnd \,|\, Z=1] = \yhat(\thb)$, as described in \cref{subsec:implementation}.

\subsection{Ordinal Regression}
We consider ordinal regression with binary outputs satisfying $y_1\ge \dots \ge y_d$. 
The output space is $\Ycal = \{\zeros, \e_1, \e_1+\e_2, \dots, \e_1+\dots+\e_d\}$. 
Since $\norm{\y - \y'}_2 \ge 1$ for any distinct $\y, \y' \in \Ycal$, $\nu = 1$ holds. 
The target loss considered in \citet{Blondel2019-gd} is the absolute loss $\norm{\y' - \y}_1$. 
Scaling it to $[0, 1]$, we can express it on $\conv(\Ycal)$ for any $\y \in \Ycal$ in the same way as the Hamming loss in \cref{subsec:example}: 
$\LT{\y'}{\y} = \frac{1}{d}\prn*{\inpr{\y', \ones} + \inpr{\y, \ones} - 2\inpr{\y', \y}}$, which is affine in $\y'$ and upper bounded by $\frac{1}{\sqrt{d}}\norm{\y' - \y}_2$, hence $\gamma = \frac{1}{\sqrt{d}}$. 
Thus, $\lambda > \frac{4\gamma}{\nu}$ holds for $d > 16$, as with the case of multilabel classification.

Since $\Ycal \subseteq \set*{0, 1}^d$ holds in this case, an efficient implementation of randomized decoding is available, as discussed in \cref{subsec:implementation}.

\paragraph{Other potential applications.}
Due to the generality of our randomized decoding, we expect that our method is potentially useful for learning tasks involving general linear optimization problems (\citealp{Elmachtoub2022-og}; \citealp{Hu2022-yq}).
Application of Fenchel--Young losses to similar tasks is studied in \citet{Berthet2020-wm}. 
We leave the investigation of this direction as future work.

\section{Lower Bound on Constant Factor in \texorpdfstring{\cref{lem:expected_target_bound}}{Lemma~\ref{lem:expected_target_bound}}}\label{asec:tightness_of_lemma}
\Cref{lem:expected_target_bound} ensures that for any $(\thb, \y) \in \R^d\times\Ycal$, the randomized decoding $\psi_\Omega$ (\cref{alg:decoding}) achieves $\E[\LT{\psi_\Omega(\thb)}{\y}] \le c \cdot \frac{\gamma}{\lambda\nu} \LSOmega{\thb}{\y}$ for $c = 4$. 
We show that the multiplicative constant $c$ cannot be smaller than $2/\ln2 \approx 2.89$ in general, suggesting \cref{lem:expected_target_bound} is nearly tight.
To see this, we again use the binary classification example in \cref{sec:randomizedq_decoding}.
Let $\Ycal = \set{\e_1, \e_2}$, $\y = \e_1$ be the ground truth, and $\thb = (\theta_1, \theta_2) = (1, 1+\ln(2^{1+\varepsilon} - 1))$ for sufficiently small $\varepsilon > 0$, which slightly favors $\e_2$ by mistake. 
As explained in \cref{subsec:example}, if $\norm{\cdot}$ is the $\ell_1$-norm, we have $\gamma = 1/2$ and $\nu = 2$. 
Also, as detailed in \cref{section:multiclass}, if $\Omega$ is the entropic regularizer given in~\eqref{eq:entropy-omega}, we have $\lambda = 1$, and the regularized prediction is given by the softmax function, i.e., $\yhat(\thb) = (1, \exp(\theta_2 - \theta_1)) / (1 + \exp(\theta_2 - \theta_1))$. 
The resulting Fenchel--Young loss is the base-$\mathrm{e}$ logistic loss, hence $\LSOmega{\thb}{\e_1} = \ln(1+\exp(\theta_2 - \theta_1))$. 
Substituting the above $\thb$ into $\yhat(\thb)$ and $\LSOmega{\thb}{\e_1}$, we obtain
\[
  \yhat(\thb) = \prn*{\frac{1}{2^{1+\varepsilon}}, 1 - \frac{1}{2^{1+\varepsilon}}} 
  \quad
  \text{and}
  \quad
  \LSOmega{\thb}{\e_1} = (1+\varepsilon) \ln2. 
\]
We then calculate the expected 0-1 loss. 
Since the closest point to $\yhat(\thb)$ in $\Ycal$ is $\ynep = \e_2$, it holds that $\Delta^* = \norm{\e_2 - \yhat(\thb)}_1 = 1/2^{\varepsilon}$, and hence $p = \min\set*{1, 2\Delta^*/\nu} = 1/2^\varepsilon$ for $\varepsilon > 0$. 
Recall that the randomized decoding returns $\y^* = \e_2$ with probability $1 - p$, or a random $\yrnd$ with probability $p$. 
Since $\yrnd$ is drawn from $\Ycal$ to satisfy $\E\brc*{\yrnd\,|\,Z=1} = \yhat(\thb)$, we have $\yrnd = \e_1$ with probability $1/2^{1+\varepsilon}$ and $\e_2$ with probability $1 - 1/2^{1+\varepsilon}$. 
Therefore, the expected 0-1 loss is 
\begin{align}
  \E[\LT{\psi_\Omega(\thb)}{\y}] 
  &= (1 - p)\LT{\e_2}{\e_1} + p\E[\LT{\yrnd}{\e_1}\,|\,Z=1]
  \\
  &= 
  \prn*{1 - \frac{1}{2^\varepsilon}} \cdot 1 
  + 
  \frac{1}{2^\varepsilon} \cdot 
  \prn*{\frac{1}{2^{1+\varepsilon}} \cdot 0 + \prn*{1 - \frac{1}{2^{1+\varepsilon}}} \cdot 1} 
  = 1 - \frac{1}{2^{1+2\varepsilon}}.
\end{align}
To ensure that \cref{lem:expected_target_bound} holds in the above setting, we need 
\[
  c \ge \frac{\lambda\nu}{\gamma} \cdot \frac{\E[\LT{\psi_\Omega(\thb)}{\y}]}{\LSOmega{\thb}{\y}} = 4 \cdot \frac{1 - \frac{1}{2^{1+2\varepsilon}}}{(1+\varepsilon) \ln2}.
\]
The right-hand side converges to $2/\ln2$ as $\varepsilon \to +0$, implying $c \ge 2/\ln2 \approx 2.89$.

\section{Bounding Surrogate Regret with Other OCO Algorithms}\label{asec:other-oco-algorithms}
While we have obtained finite surrogate regret bounds using OGD with the constant learning rate, it may not perform well in practice since it does not utilize information from past rounds. 
Below, we demonstrate that we can achieve finite surrogate regret bounds with more practical OCO algorithms.
Specifically, \cref{asubsec:ogd-adaptive,asec:parameter_free} discuss using OGD with an adaptive learning rate and a parameter-free OCO algorithm, respectively.
In terms of theoretical guarantees, however, the results presented below are weaker than those obtained in the main text by using OGD with the constant learning rate: 
the result in \cref{asubsec:ogd-adaptive} is not parameter-free, and the surrogate regret bound in \cref{asec:parameter_free} is asymptotically larger by a logarithmic factor and does not vanish when $a \to 1$.

\subsection{OGD with Adaptive Learning Rate}\label{asubsec:ogd-adaptive}
Let $a = 1 - \frac{4\gamma}{\lambda\nu}$ and $b = \frac{2\CC\kappa^2}{\lambda}$, as in the proof of \cref{thm:expected_regret_general}.
Assume that a domain $\Wcal$ with an $\ell_2$-diameter of $B > 0$ is given. 
We consider using the online gradient descent on $\Wcal$ with an adaptive learning rate as $\mathsf{Alg}$, where the learning rate in the $t$th round is set to ${B}/{\sqrt{2\sum_{i=1}^T \norm{\nabla\S_i(\W_i)}_\mathrm{F}^2}}$ (\citealp{McMahan2010-un}; \citealp{Duchi2011-kr}). 
Due to $\norm{\nabla\S_t(\W_t)}_\mathrm{F}^2 \le b\S_t(\W_t)$, this OGD achieves the following $L^\star$ bound for any $\U \in \Wcal$ \citep[Theorem~4.25]{orabona2023modern}: 
\begin{equation}\label{eq:ogd-regret}
  \sum_{t=1}^T (\S_t(\W_t) - \S_t(\U)) \le 2bB^2 + 2B\sqrt{2b\sum_{t=1}^T\S_t(\U)}.
\end{equation}
If we learn $\W_t$ with this OGD, we can achieve the following expected surrogate regret bound.

\begin{theorem}\label{thm:expected_regret_adaptive_ogd}
  Let $\psi_\Omega$ be the randomized decoding given in \cref{alg:decoding} and $C > 0$ a constant such that $\max_{t\in[T]}\norm{\x_t}_2 \le C$. 
  If $\mathsf{Alg}$ satisfies \eqref{eq:ogd-regret} and $\lambda > \frac{4\gamma}{\nu}$ holds, for any $\U \in \Wcal$, it holds that 
  \begin{align}
  \sum_{t=1}^T \E_t[L_t(\yprd_t)] 
  \le 
  \sum_{t=1}^T \S_t(\U) + 
  \prn*{
    1 - \frac{4\gamma}{\lambda\nu}
  }^{-1}
  \frac{16\gamma\CC\kappa^2B^2}{\lambda^2\nu}
  \end{align}
\end{theorem}
\begin{proof}
  Since we have $\E_t[L_t(\yprd_t)] \le (1 - a)\S_t(\W_t)$ due to \cref{lem:expected_target_bound}, it holds that 
  \[
    \sum_{t=1}^T \E_t[L_t(\yprd_t)] - \sum_{t=1}^T \S_t(\U)
    \le 
    (1-a)\sum_{t=1}^T \prn*{
        \S_t(\W_t) - \S_t(\U)
      }
      - a\sum_{t=1}^T \S_t(\U).
  \]
  Substituting \eqref{eq:ogd-regret} into the right-hand side implies that the surrogate regret is at most 
  \[
    \begin{aligned}
      2(1-a)bB^2 + 2(1-a)B\sqrt{2b\sum_{t=1}^T\S_t(\U)} - a\sum_{t=1}^T \S_t(\U)
      &\le
      2(1-a)bB^2 + \frac{2(1-a)^2bB^2}{a}
      \\
      &=
      \frac{2(1-a)bB^2}{a},
    \end{aligned}
  \]
   where the inequality comes from $\sqrt{c_1 x} - c_2 x \le \frac{c_1}{4c_2}$ ($\forall x > 0$) for $c_1, c_2 > 0$.
   The desired surrogate regret bound follows from $a = 1 - \frac{4\gamma}{\lambda\nu}$ and $b = \frac{2\CC\kappa^2}{\lambda}$.
\end{proof}
Similar to the bound in \cref{prop:surrogate_gap}, it vanishes when $a \to 1$.
However, as described above, tuning the learning rate requires the knowledge of the domain size, $B$, hence no longer parameter-free.

\subsection{Parameter-Free Algorithm}\label{asec:parameter_free}
In the previous section, we have assumed that the $\ell_2$-diameter $B$ of the domain $\Wcal$ is known a priori.
In practice, however, we rarely know the precise size of $\Wcal$ containing the best estimator $\U$ in hindsight. 
A common workaround is to set $B$ to a sufficiently large value, but this typically results in overly pessimistic regret bounds. 
\emph{Parameter-free} algorithms 
(\citealp{Mcmahan2012-tf}; 
\citealp{Orabona2013-ng}; 
\citealp{Cutkosky2018-bf}) 
are designed to avoid this issue by automatically adapting to the norm of $\U$.
While OGD with the constant learning rate is also parameter-free in our case, parameter-free algorithms studied in this context would be more practical.

We consider using a parameter-free algorithm of \citet[Section~3]{Cutkosky2018-bf} as $\mathsf{Alg}$ instead of OGD. 
Their algorithm enjoys a regret bound depending on $g_{1:T} \coloneqq \sum_{t=1}^T \norm{\nabla\S_t(\W_t)}_\mathrm{F}^2$, which is helpful in the subsequent analysis. 
Specifically, the bound on $\sum_{t=1}^T (\S_t(\W_t) - \S_t(\U))$ is 
\begin{equation}
  \begin{aligned}
  O \prn[\Bigg]{
  \norm{\U}_\mathrm{F} 
  \max \Bigg\{
  \sqrt{
    g_{1:T} \,
    \ln \left( \frac{\norm{\U}_\mathrm{F}^2 g_{1:T}}{\epsilon^2} + 1 \right)
  }
  ,\,
  \ln \frac{\norm{\U}_\mathrm{F} g_{1:T} }{\epsilon}
  \Bigg\}
  +
  \norm{\U}_\mathrm{F} \sqrt{ g_{1:T} }
  + 
  \epsilon
  }
  ,
  \end{aligned}  
\end{equation}
where $\epsilon > 0$ is an \emph{initial-wealth} parameter specified by the user. 
We let $\epsilon = 1$ for simplicity. 
Then, omitting lower-order terms, the regret bound of $\mathsf{Alg}$ reduces to 
\begin{equation}\label{eq:parameterfree-regret}
  \begin{aligned}
  O \prn[\Bigg]{
  \norm{\U}_\mathrm{F} 
  \sqrt{
    g_{1:T} \,
    \ln \left( \norm{\U}_\mathrm{F}^2 g_{1:T} + 1 \right)
  }
  }
  .
  \end{aligned}  
\end{equation}
Compared with the standard regret bound of OGD, the dependence on $g_{1:T}$ is worse by a factor of $\sqrt{\ln g_{1:T} }$, which is inevitable in parameter-free learning (\citealp{Mcmahan2012-tf}; \citealp{Orabona2013-ng}).
This makes the analysis for exploiting the surrogate gap a bit trickier.
The next lemma offers a helpful inequality for exploiting the surrogate gap with the parameter-free regret bound~\eqref{eq:parameterfree-regret}.
\begin{lemma}\label[lemma]{lem:general_sb_direct}
  Let $a, b, c > 0$. For any $x \ge 0$, it holds that 
  \begin{equation}
    - ax + \sqrt{bx \ln(cx+1)} 
    \leq
    \frac{b}{2a}
    \prn*{
      \ln \prn*{ \frac{b}{2a^2} + 1}
      + 
      \ln (c + 1)
    }
    + a
    .
    \label{eq:general_sb_direct}
  \end{equation}
\end{lemma}
\begin{proof}
  Let $f(x) = - ax + \sqrt{bx \ln(cx+1)}$.
  Since $f$ is continuous and goes to $-\infty$ as $x \to \infty$, it suffices to show that $f$ is bounded as in the lemma statement at any critical point $x^*$.
  The derivative of $f$ is
  $f'(x) = - a + \frac{\sqrt{b}}{2} \frac{\frac{cx}{cx+1} + \ln(cx+1)}{\sqrt{x \ln(cx+1)}}$, and hence any critical point $x^*$ satisfies
  $
  \sqrt{bx^* \ln(cx^*+1)}
  =
  \frac{b}{2a}
  \prn*{
    \frac{cx^*}{cx^*+1} + \ln(cx^*+1)
  }.
  $
  Thus, we have
  \begin{align}\label{eq:fstar}
    f(x^*)
    =
    - a x^*
    +
    \frac{b}{2a}
    \prn*{
      \frac{cx^*}{cx^*+1} + \ln(cx^*+1)
    }
    \leq 
    - a x^*
    +
    \frac{b}{2a}
    +
    \frac{b}{2a} \ln(cx^*+1).
  \end{align}
  If $c \le 1$, the right-hand side of \eqref{eq:fstar} is at most 
  $- a x^*
  +
  \frac{b}{2a}
  +
  \frac{b}{2a} \ln(x^*+1)$, and inspecting its derivative with respect to $x^*$ readily shows that the following inequality holds for any $x^* \ge 0$:
  \begin{equation}\label{eq:lin-log-ineq}
    - a x^*
    +
    \frac{b}{2a}
    +
    \frac{b}{2a} \ln(x^*+1)
    \le 
    \frac{b}{2a}\max\set*{\ln \prn*{ \frac{b}{2a^2} }, 0} + a.
  \end{equation}
  If $c > 1$, the right-hand side of \eqref{eq:fstar} is at most 
  $- a x^*
  +
  \frac{b}{2a}
  +
  \frac{b}{2a} \ln(x^*+1)
  + 
  \frac{b}{2a} \ln c
  $, 
  and the sum of the first three terms is bounded as in \eqref{eq:lin-log-ineq}.
  Thus, $f(x^*) \le \frac{b}{2a}\max\set*{\ln \prn*{ \frac{b}{2a^2} }, 0} + \frac{b}{2a}\max\set*{\ln \prn*{ c}, 0} + a$ holds in any case. 
  By using $\max\set{\ln z, 0} \le \ln(z + 1)$ for $z > 0$, we obtain the desired bound.  
\end{proof}
Now, we are ready to obtain a finite surrogate regret bound with the parameter-free algorithm.
\begin{theorem}\label[theorem]{thm:expected_regret_general_parameter_free}
  Suppose $\psi_\Omega$ and $C$ to be given as in \cref{thm:expected_regret_general}. 
  Fix any $\U\in\Wcal$, where $\Wcal$ is some domain.
  Let $\mathsf{Alg}$ be the parameter-free algorithm that achieves the regret bound given in~\eqref{eq:parameterfree-regret} without knowing the size of $\Wcal$ (or $\norm{\U}_\mathrm{F}$). 
  If $\lambda > \frac{4\gamma}{\nu}$, it holds that 
  \begin{equation}
  \sum_{t=1}^T \E_t[L_t(\yprd_t)] 
  = 
  \sum_{t=1}^T \S_t(\U) 
  +
  O\prn*{
    \frac{\norm{\U}_\mathrm{F}^2 \CC \kappa^2}{ \lambda - {4\gamma}/{\nu}}
      \ln\prn*{
      \frac{\norm{\U}_\mathrm{F}^2 C^2 \kappa^2\lambda}{\prn*{\lambda - {4\gamma}/{\nu}}^2}
      + 1
      }
  }
  .
  \end{equation}
\end{theorem}
\begin{proof} 
Let $a = 1 - \frac{4\gamma}{\lambda\nu} \in (0, 1)$ and $x = \sum_{t=1}^T \S_t(\W_t)$.
As in the proof of \Cref{thm:expected_regret_general}, we have
\begin{equation}\label{eq:surr_gap}
  \sum_{t=1}^T \prn*{\E_t[L_t(\yprd_t)] - \S_t(\W_t)} \le - ax
\end{equation}
and
\begin{equation}\label{eq:grad2surr}
  g_{1:T} = \sum_{t=1}^T \norm{\nabla\S_t(\W_t)}_\mathrm{F}^2 \le \frac{2\CC\kappa^2}{\lambda}x
  .
\end{equation} 
Substituting \eqref{eq:grad2surr} into the regret bound \eqref{eq:parameterfree-regret} and using \eqref{eq:surr_gap}, we can bound the surrogate regret as follows: 
\begin{equation}\label{eq:last_parameterfree}
  \begin{aligned}
  \sum_{t=1}^T \E_t[L_t(\yprd_t)] 
  -
  \sum_{t=1}^T \S_t(\U) 
  ={}&
  \sum_{t=1}^T \prn*{\E_t[L_t(\yprd_t)] - \S_t(\W_t)} + \sum_{t=1}^T \prn*{\S_t(\W_t) - \S_t(\U)} 
  \\
  \le{}&
  - ax
  +
  c_\mathsf{Alg}
  \norm{\U}_\mathrm{F} 
  \sqrt{
    \frac{2\CC\kappa^2}{\lambda}x \,
    \ln \left( \norm{\U}_\mathrm{F}^2 \frac{2\CC\kappa^2}{\lambda}x + 1 \right)
  }
  ,
  \end{aligned}  
\end{equation}
where $c_\mathsf{Alg} > 0$ is an absolute constant hidden in the big-O notation in \eqref{eq:parameterfree-regret}.
By using \cref{lem:general_sb_direct} with $a = 1 - \frac{4\gamma}{\lambda\nu}$, $b = c_\mathsf{Alg}^2 \norm{\U}_\mathrm{F}^2\frac{2\CC\kappa^2}{\lambda}$, and $c = \norm{\U}_\mathrm{F}^2 \frac{2\CC\kappa^2}{\lambda}$, we can upper bound the surrogate regret by
\begin{align}
  &\frac{b}{2a}
    \prn*{
      \ln \prn*{ \frac{b}{2a^2} + 1}
      + 
      \ln (c + 1)
    }
    + a
  \\
  ={}&
  O\prn*{
    \frac{\norm{\U}_\mathrm{F}^2 \CC \kappa^2}{ \lambda - \frac{4 \gamma}{\nu}}
    \prn*{
      \ln\prn*{
      \frac{\norm{\U}_\mathrm{F}^2 C^2 \kappa^2}{\lambda\prn*{1 - \frac{4\gamma}{\lambda\nu}}^2}
      + 1
      }
      +
      \ln\prn*{
      \frac{\norm{\U}_\mathrm{F}^2 C^2 \kappa^2}{\lambda}
      + 1
      }
    }
  }
  \\
  \lesssim{}&
  O\prn*{
    \frac{\norm{\U}_\mathrm{F}^2 \CC \kappa^2}{ \lambda - \frac{4 \gamma}{\nu}}
      \ln\prn*{
      \frac{\norm{\U}_\mathrm{F}^2 C^2 \kappa^2\lambda}{\prn*{\lambda - \frac{4 \gamma}{\nu}}^2}
      + 1
      }
  }
  ,
\end{align}
where the asymptotic inequality is due to $1 - \frac{4\gamma}{\lambda\nu} < 1$.
Thus, we obtain the desired bound.
\end{proof}
Note that the bound in \cref{thm:expected_regret_general_parameter_free} depends on $\norm{\U}_\mathrm{F}$ and is free from $B$, in contrast to the bound in \cref{thm:expected_regret_adaptive_ogd}. 
However, the bound does not vanish when $a = 1 - \frac{4\gamma}{\lambda\nu}$ approaches $1$.

\section{Proof of \texorpdfstring{\cref{thm:high_probability_regret_general}}{Theorem~\ref{thm:high_probability_regret_general}}}\label{asec:high_probability_regret_general}

We prove the high-probability surrogate regret bound.
Recall that $Z_t$ for $t = 1,\dots,T$ are random variables defined by $Z_t \coloneqq L_t(\yprd_t) - \E_t[L_t(\yprd_t)]$.
\begin{proof}
  Let $a = 1 - \frac{4\gamma}{\lambda\nu} \in (0, 1)$ and $b = \frac{2\CC\kappa^2}{\lambda} > 0$, as in the proof of \cref{thm:expected_regret_general}.
  We decompose the surrogate regret, $\sum_{t=1}^T L_t(\yprd_t) - \sum_{t=1}^T \S_t(\U)$, as follows: 
  \begin{equation}\label{eq:high_probability_decomposition}
    \sum_{t=1}^T \prn*{L_t(\yprd_t) - \E_t[L_t(\yprd_t)]} + \sum_{t=1}^T \prn*{\E_t[L_t(\yprd_t)] - \S_t(\U)}. 
  \end{equation}
  Let $\eta' \coloneqq \frac{\eta b}{2} = \frac{a}{2} < 1$.
  As in the proof of \cref{prop:surrogate_gap}, the second term in \eqref{eq:high_probability_decomposition} is at most
  \begin{equation}\label{eq:expected_term_bound}
    \begin{aligned}
      &
      (1-a)\prn*{1 - \frac{\eta b}{2}}^{-1}\frac{\norm{\U}_\mathrm{F}^2}{2\eta}
      -
      \prn*{
        a
        -
        (1-a)\prn*{1 - \frac{\eta b}{2}}^{-1}
        \frac{\eta b}{2}
      }      
      \sum_{t=1}^T \S_t(\U)
      \\
      ={}&
      \frac{1 - a}{1 - \eta'} \cdot \frac{b\norm{\U}_\mathrm{F}^2}{4\eta'} 
      - 
      \prn*{a - \frac{1 - a}{1 - \eta'}\cdot\eta'} \sum_{t=1}^T \S_t(\U). 
    \end{aligned}
  \end{equation}
  Below, we derive an upper bound on the first term in \eqref{eq:high_probability_decomposition}.
  Note that the first term equals $\sum_{t=1}^T Z_t$.
  Since \cref{assump:nu} implies $0 \le L_t(\y) \le \gamma D$ for any $\y \in \Ycal$, we have $|Z_t| \le \gamma D$, hence 
  $\E_t[Z_t^2] \le 
  \E_t\brc*{L_t(\yprd_t)^2} 
  \le 
  \gamma D \E_t\brc*{L_t(\yprd_t)}$. 
  Combining this with \cref{lem:expected_target_bound} yields $\E_t[Z_t^2] \le \frac{4\gamma^2D}{\lambda\nu} \S_t(\W_t) = (1 - a)\gamma D \S_t(\W_t)$. 
  By applying Bernstein's inequality (see, e.g., \citet[Lemma~A.8]{Cesa-Bianchi2006-oa}) to the bounded martingale difference sequence $Z_1,\dots,Z_T$ with bounded variance, for any $\delta \in (0, 1)$, with probability at least $1 - \delta$, it holds that
  \begin{equation}\label{eq:high_probability_deviation_bound}
  \sum_{t=1}^T Z_t 
  \le 
  \sqrt{2(1 - a)\gamma D \sum_{t=1}^T\S_t(\W_t)\ln\frac{1}{\delta}} + \frac{\sqrt{2}}{3}\gamma D\ln\frac{1}{\delta}
  =
  \sqrt{2\zeta(1 - a) \sum_{t=1}^T\S_t(\W_t)} + \frac{\sqrt{2}}{3}\zeta,
  \end{equation}
  where we let $\zeta = \gamma D \ln\frac{1}{\delta}$ for simplicity.
  Also, the $L^\star$ bound in the proof of \cref{prop:surrogate_gap} implies 
  \[
    \sum_{t=1}^T \S_t(\W_t)
    \le 
    \prn*{1 - \frac{\eta b}{2}}^{-1}
    \prn*{
    \frac{\norm{\U}_\mathrm{F}^2}{2\eta} + \sum_{t=1}^T \S_t(\U)
    }
    =
    \frac{1}{1 - \eta'}
    \cdot
    \frac{b\norm{\U}_\mathrm{F}^2}{4\eta'} + 
    \frac{1}{1 - \eta'}
    \cdot
    \sum_{t=1}^T \S_t(\U)
    .
  \]
  Substituting this into \eqref{eq:high_probability_deviation_bound} and using the subadditivity of $\sqrt{\cdot}$ imply that the first term in \eqref{eq:high_probability_decomposition} is at most 
  \begin{equation}\label{eq:deviation_term_bound}
    \sqrt{
      \frac{\zeta(1 - a)}{1 - \eta'}
      \cdot
      \frac{b\norm{\U}_\mathrm{F}^2}{2\eta'}
    }
    +
    \sqrt{
      \frac{2\zeta(1 - a)}{1 - \eta'}
      \cdot
      \sum_{t=1}^T \S_t(\U)
    }
    + \frac{\sqrt{2}}{3}\zeta.
  \end{equation}
  Therefore, the surrogate regret is bounded from above by the sum of \eqref{eq:expected_term_bound} and \eqref{eq:deviation_term_bound}.
  Since $\eta' < a$, the coefficient of the second term in \eqref{eq:expected_term_bound} satisfies 
  $a - \frac{1 - a}{1 - \eta'}\cdot \eta' > 0$.
  From $\sqrt{c_1 x} - c_2 x \le \frac{c_1}{4c_2}$ ($\forall x > 0$) for $c_1, c_2 > 0$, we can offset the middle term in \eqref{eq:deviation_term_bound} with the second term in \eqref{eq:expected_term_bound} as follows: 
  \[
    \sqrt{
      \frac{2\zeta(1 - a)}{1 - \eta'}
      \cdot
      \sum_{t=1}^T \S_t(\U)
    }
    -
      \prn*{
        a
        -
        \frac{1 - a}{1 - \eta'}
        \cdot
        \eta'
      }      
      \sum_{t=1}^T \S_t(\U)
      \le
      \frac{\zeta}{2}
      \cdot
      \frac{
        1 - a
      }{
        a - \eta'
      }
      =
      \frac{1 - a}{a}\zeta.
  \]
  By putting everything together, the surrogate regret is at most
  \[
    \frac{1 - a}{1 - \eta'} \cdot \frac{b\norm{\U}_\mathrm{F}^2}{4\eta'} 
    +
    \sqrt{
      \frac{\zeta(1 - a)}{1 - \eta'}
      \cdot
      \frac{b\norm{\U}_\mathrm{F}^2}{2\eta'}
    }
    + 
    \frac{\sqrt{2}}{3}\zeta
    +
    \frac{1 - a}{a}\zeta
    \le 
    \frac1a\prn*{
      (1-a)b\norm{\U}_\mathrm{F}^2 + \zeta
    },
  \]
  where the inequality is for simplification; 
  we applied AM--GM to the second term to upper bound it by $\frac\zeta2 + \frac{1 - a}{1 - \eta'} \cdot \frac{b\norm{\U}_\mathrm{F}^2}{4\eta'}$, used $\frac{1}{1 - \eta'} \le 2$ (since $\eta' = \frac{a}{2} < \frac12$), and simplified the resulting coefficient of $\zeta$.
  Substituting $a = 1 - \frac{4\gamma}{\lambda\nu}$, $b = \frac{2\CC\kappa^2}{\lambda}$, and $\zeta = \gamma D \ln\frac{1}{\delta}$ into it completes the proof. 
\end{proof}

\section{Missing Details of Online-to-Batch Conversion}\label{asec:online-to-batch}
We prove the offline learning guarantee via online-to-batch conversion. 
We also discuss its connection to excess risk bounds.

\subsection{Proof of \texorpdfstring{\cref{thm:online-to-batch}}{Theorem~\ref{thm:online-to-batch}}}\label{asec:online-to-batch-proof}
\begin{proof}
  Let $a = 1 - \frac{4\gamma}{\nu\lambda}$ and $b = \frac{2\CC\kappa^2}{\lambda}$, as in the proof of \cref{thm:expected_regret_general}.
  For any $(\x, \y) \in \Xcal\times\Ycal$, \cref{lem:expected_target_bound} ensures $\E[\LT{\psi_\Omega(\overline{\W}\x)}{\y}] \le (1 - a)\LSOmega{\overline{\W}\x}{\y}$, where the expectation is about the randomness of $\psi_\Omega$. 
  Let $\rho$ denote the joint distribution on $\Xcal\times\Ycal$. 
  Since $\LSOmega{\W\x}{\y}$ is convex in $\W$, Jensen's inequality implies that the expected target loss, $\E[\LT{\psi_\Omega(\overline{\W}\x)}{\y}]$, is bounded as follows:
\begin{align}
  \mathop{\E}_{\x,\y\sim\rho}\brc*{
    \E\brc*{
      \LT{\psi_\Omega(\overline{\W}\x)}{\y}
    \mid
      \x, \y
    }
  }  
  &\le 
  (1 - a)
  \mathop{\E}_{\x,\y\sim\rho}\brc*{
    \LSOmega{\overline{\W}\x}{\y}
  }
  \\
  &\le
  (1 - a)\cdot\frac{1}{T}\sum_{t=1}^T\mathop{\E}_{\x,\y\sim\rho}[\LSOmega{\W_t\x}{\y}]. 
\end{align}
Furthermore, since $\W_t$ depends only on $(\x_s, \y_s)_{s<t}$, the law of total expectation implies $\E[\S_t(\W_t)] = \E[\LSOmega{\W_t\x}{\y}]$ (see, e.g., \citet[Theorem~3.1]{orabona2023modern} for a similar discussion); also, $\E[\S_t(\U)] = \E[\LSOmega{\U\x}{\y}]$ holds for any fixed $\U \in \Wcal$.
Therefore, it holds that
\begin{align}
  &\E[\LT{\psi_\Omega(\overline{\W}\x)}{\y}] - \E[\LSOmega{\U\x}{\y}] 
  \\
  \le{}& (1 - a)\cdot\frac{1}{T}\sum_{t=1}^T\E[\LSOmega{\W_t\x}{\y}] - \frac{1}{T}\sum_{t=1}^T\E[\LSOmega{\U\x}{\y}] \\
  ={}& (1 - a)\cdot\frac{1}{T}\sum_{t=1}^T\E[\S_t(\W_t)] - \frac{1}{T}\sum_{t=1}^T\E[\S_t(\U)] \\
  ={}& \frac{1}{T}\cdot\E\brc*{
      (1 - a)\sum_{t=1}^T\prn*{\S_t(\W_t) - \S_t(\U)} - a \sum_{t=1}^T \S_t(\U)
  }
  \\
  \le{}& \frac{1}{T}\cdot 
  (1-a)\prn*{1 - \frac{\eta b}{2}}^{-1}\frac{\norm{\U}_\mathrm{F}^2}{2\eta}.
\end{align}
where the last inequality follows from the same discussion as that in the proof of \cref{prop:surrogate_gap}.
By substituting $a = 1 - \frac{4\gamma}{\nu\lambda}$ and $b = \frac{2\CC\kappa^2}{\lambda}$ into the right-hand side, we obtain the desired bound.
\end{proof}

\subsection{Connection to Excess Risk Bound}\label{asec:comparison_offline}
We derive fast convergence of the excess risk from the surrogate regret bound in \cref{thm:online-to-batch} under separability and realizability assumptions detailed below. 
The excess risk that we aim to bound is
\[
  \E[\LT{\psi_\Omega(\overline{\W}\x)}{\y}] - \E[\LT{\y_0}{\y}],
\]
where $\y_0$ is the Bayes rule.
As \cref{thm:online-to-batch} bounds the first term $\E[\LT{\psi_\Omega(\overline{\W}\x)}{\y}]$ by $\E[\LSOmega{\U\x}{\y}]$, plus an asymptotically vanishing term of $O(1/T)$, it is sufficient to show 
\begin{enumerate}
  \item that $\y_0$ can be decoded from some linear predictor $\U_0$ and satisfies $\E[\LT{\y_0}{\y}] = 0$; and
  \item that $\E[\LSOmega{\U\x}{\y}] = 0$ holds for some linear estimator $\U$.
\end{enumerate}
Consequently, these combined with \cref{thm:online-to-batch} imply the excess risk bound of
\[
  \E[\LT{\psi_\Omega(\overline{\W}\x)}{\y}] - \E[\LT{\y_0}{\y}] = O(1/T).
\]
Below, we design a Bayes rule (or, the best linear estimator and decoder) and check $\E[\LT{\y_0}{\y}] = 0$, and we confirm $\E[\LSOmega{\U\x}{\y}] = 0$ by using the SparseMAP surrogate loss $\S_\Omega$. 

Let us introduce definitions needed in the subsequent discussion.
We define the \emph{frontier} (or the decision boundary) in a similar manner to \citet{Cabannes2021-vv} as follows: 
\[
  F \coloneqq \Set*{\thb \in \R^d}{\abs{\argmax_{\y \in \Ycal}\inpr{\thb, \y}} \ge 2}. 
\] 
We also define the normal cone $N(\y)$ at $\y \in \Ycal$ and its boundary $E(\y)$ by
\begin{align}
  N(\y) &\coloneqq \Set*{\thb \in \R^d}{\forall\y' \in \conv(\Ycal), \inpr{\thb, \y' - \y} \le 0} \quad \text{and}\\
  E(\y) &\coloneqq \bigcup_{\y' \in \Ycal\setminus\set{\y}}\Set*{\thb \in N(\y)}{\inpr{\thb, \y' - \y} = 0},  
\end{align} 
respectively.
One can readily confirm $F = \bigcup_{\y \in \Ycal}E(\y)$.
Let $d(\thb, F) \coloneqq \min\Set*{\norm{\thb - \thb'}_2}{\thb' \in F}$ denote the distance from $\thb$ to $F$ and  
$D$ the $\ell_2$-diameter of $\conv(\Ycal)$. 
Below is a fundamental fact.
\begin{lemma}\label[lemma]{lem:cone}
  Let $\y \in \Ycal$ and $\thb \in N(\y)$. 
  If $d(\thb, F) \ge t$ holds for some $t > 0$, we have $\inpr{\thb, \y - \y'} \ge t\norm{\y - \y'}_2$ for any $\y' \in \Ycal$.  
\end{lemma}
\begin{proof}
  The case of $\y = \y'$ is trivial. 
  Below, we shift $\y$ to $\zeros$ and prove $\inpr{\thb, \bm{z}} \le -t\norm{\bm{z}}_2$ for any $\bm{z} \in \Ycal\setminus\set{\y'} - \set{\y}$; let $N = N(\zeros)$ and $E = E(\zeros)$. 
  Since $\bm{z} \notin N$, there exists $\bm{\xi} \in E\setminus\set{\zeros}$ such that the sum of the angles between $\thb$ and $\bm{\xi}$ and between $\bm{\xi}$ and $\bm{z}$, denoted by $\alpha,\beta\ge0$, respectively, equals the angle between $\thb$ and $\bm{z}$ and $\alpha + \beta \le \pi$ holds. 
  Due to $\bm{\xi} \in E \subseteq N$, we have $\beta \ge \pi/2$.
  Furthermore, since $d(\thb, F) \ge t$ implies $d(\thb, E) \ge t$, we have $\sin \alpha \ge t/\norm{\thb}_2$. 
  Thus, $\cos(\alpha + \beta) \le \cos(\alpha + \pi/2) = -\sin\alpha \le -t/\norm{\thb}_2$ holds, hence $\inpr{\thb, \bm{z}} = \norm{\thb}_2\norm{\bm{z}}_2\cos(\alpha + \beta) \le -t\norm{\bm{z}}_2$.
\end{proof}

\paragraph{Assumptions.} 
We assume a variant of the margin condition introduced in \citet[Assumption~3]{Cabannes2021-vv} and a realizability condition. 
Specifically, we assume there exists a linear estimator $\U_0:\Xcal\to\R^d$ satisfying the \emph{no-density separation} for some margin $t_0 > 0$, i.e., 
\[
  \mathbb{P}_\Xcal\prn*{{d(\U_0\x, F) < t_0}} = 0,
\]
where $\mathbb{P}_\Xcal$ represents the probability with respect to the marginal distribution of $\rho$ on $\Xcal$.
In binary classification, this is sometimes called \emph{Massart's noise condition}. 
We also require the underlying distribution $\rho$ to satisfy a realizability condition that the above $\U_0$ satisfies 
\[
  \mathbb{P}(\phi(\U_0\x) = \y) = 1,
\]
where $\phi:\R^d\to\Ycal$ is a decoding function given by $\phi(\U_0\x) = \argmax\Set*{\inpr{\U_0\x, \y}}{\y\in \Ycal}$ (ties occur with probability zero due to the no-density separation).
Then, $\x \mapsto \phi(\U_0\x) \in \Ycal$ is the Bayes rule that attains the zero target risk, i.e., $\E[\LT{\y_0}{\y}] = \E[\LT{\phi(\U_0\x)}{\y}] = 0$. 
Therefore, the first condition is confirmed.

\paragraph{Surrogate loss.}
To establish the second condition $\E[\LSOmega{\U\x}{\y}] = 0$, we need a surrogate loss that attains zero for well-separated data. 
To this end, we employ the SparseMAP loss, which is known to have a unit \emph{structured separation margin} (see \citet[Section~7.4]{Blondel2020-tu}). 
Specifically, for any $(\thb, \y) \in \R^d\times\Ycal$, we have $\LSOmega{\thb}{\y} = 0$ if 
\begin{equation}\label{eq:structured_separation_margin}
  \inpr{\thb, \y} \ge \max\Set*{\inpr{\thb, \y'} + \frac{1}{2}\norm{\y - \y'}_2^2}{\y' \in \Ycal}.
\end{equation}
We show that $\U = \frac{D}{2t_0}\U_0$ satisfies $\E[\LSOmega{\U\x}{\y}] = 0$.
Due to the assumptions and \cref{lem:cone} with $\U_0\x \in N(\phi(\U_0\x))$, for $(\x, \y)$ drawn from $\rho$ and any $\y' \in \Ycal$, with probability $1$, we have
\begin{equation}
  \inpr{\U\x, \y} 
  = \frac{D}{2t_0}\inpr{\U_0\x, \phi(\U_0\x)} 
  \ge \frac{D}{2t_0}\inpr{\U_0\x, \y'} + \frac{D}{2}\norm{\phi(\U_0\x) - \y'}_2 
  \ge \inpr{\U\x, \y'} + \frac{1}{2}\norm{\y - \y'}_2^2,
\end{equation}
which implies \eqref{eq:structured_separation_margin} and thus $\E[\LSOmega{\U\x}{\y}] = 0$ holds.

\begin{remark}
  As mentioned in \cref{subsec:onlinetobatch}, \citet{Cabannes2021-vv} has already achieved an excess risk bound with an exponential rate under weaker assumptions, which is faster than the above $O(1/T)$ rate. 
  Hence, the purpose of the above discussion is to deepen our understanding of the relationship between the surrogate regret and excess risk bounds, rather than to present a novel result.
\end{remark}

\section{Proof of \texorpdfstring{\cref{thm:lower_bound}}{Theorem~\ref{thm:lower_bound}}}\label{asec:lower_bound}
\begin{proof}
  For simplicity, assume that $M \coloneqq (B^2 - \ln^2(dT))/\ln^2(2d)$ is a positive integer. 
  We sample true class $i_t \in [d]$ uniformly at random for $t=1,\dots,M+1$. 
  For $t > M+1$, we set $i_t = i_{M+1}$. 
  Each $\x_t$ is a vector of length $M+1$.
  For $t = 1,\dots,M+1$, we let $\x_t = \e_t$, the $t$th standard basis vector in $\R^{M+1}$. 
  For $t > M+1$, we let $\x_t = \e_{M+1}$.
  We define an offline estimator $\U' \in \R^{d\times (M+1)}$ as follows:
  the $t$th column of $\U'$ is $\ln(2d) \e_{i_t}$ for $t = 1,\dots,M$, and the ($M+1$)th column is $\ln(dT)\e_{M+1}$. 
  Note that $\norm{\U'}_\mathrm{F}^2 = M\ln^2(2d) + \ln^2(dT) = B^2$ always holds.
  Fix any learner's algorithm. 
  For the first $M$ rounds, the logistic loss of $\U'$ is bounded as $-\log_2\frac{2d}{2d + d-1} = \log_2\prn*{1 + \frac12\prn*{1 - \frac{1}{d}}} \le \frac12\prn*{1 - \frac{1}{d}}$. 
  Since each $i_t \in [d]$ is sampled uniformly at random, the expected 0-1 loss is $1 - \frac{1}{d}$. 
  Therefore, the expected surrogate regret summed over the first $M$ rounds is at least
  \[
    \sum_{t=1}^M \E[\ind_{\yprd_t \neq \y_t}] 
    - 
    \sum_{t=1}^M \S_t(\U')
    \ge 
    M\prn*{1 - \frac{1}{d}} - \frac{M}{2}\prn*{1 - \frac{1}{d}}
    =
    \frac{M}{2}\prn*{1 - \frac{1}{d}}
    \ge 
    \frac{M}{4}
    =
    \Omega\prn*{\frac{B^2}{\ln^2 d}},
  \]
  where we used $d \ge 2$ and $B = \Omega(\ln(dT))$. 
  As for the remaining $T - M$ rounds, the logistic loss value is at most $\frac{1}{T}\prn*{1 - \frac{1}{d}}$ by a similar calculus to the above, whereas the expected 0-1 loss is at least $1 - \frac{1}{d}$ since $i_{M+1}$ is sampled uniformly at random. 
  Therefore, the expected surrogate regret over the $T - M$ rounds is non-negative. 
  In total, the expected surrogate regret is $\Omega\prn*{\frac{B^2}{\ln^2 d}}$.
\end{proof}

\end{document}